%% file: main.tex
\documentclass[sigconf]{acmart}

\AtBeginDocument{%
  \providecommand\BibTeX{{%
    \normalfont B\kern-0.5em{\scshape i\kern-0.25em b}\kern-0.8em\TeX}}}

\usepackage{amsmath,amsthm}
\usepackage{bm}
\usepackage[short]{optidef}
\usepackage{algpseudocode}
\usepackage{caption}
\usepackage{algorithm}
\usepackage{graphicx}
\usepackage{subfigure}

\usepackage{bbm}
\usepackage[short]{optidef}
\usepackage{tabularx,ragged2e}
\newtheorem{lemma}{Lemma}
\newtheorem{theorem}{Theorem}
\newtheorem{proposition}{Proposition}
\newtheorem{assumption}{Assumption}
\newtheorem{corollary}{Corollary}
\theoremstyle{definition}
\newtheorem{remark}{Remark}
\input{math_def}

\usepackage{thmtools,thm-restate}

\usepackage{color,enumitem}
\newboolean{showcomments}
\setboolean{showcomments}{true}

\setcopyright{none}
\renewcommand\footnotetextcopyrightpermission[1]{}

\newcommand{\Junshan}[1]{  \ifthenelse{\boolean{showcomments}}
	{ \textcolor{red}{(Junshan says:  #1)}} {}  }

\begin{document}

\title{Accelerating Distributed Online Meta-Learning via Multi-Agent Collaboration under Limited Communication}

\author{Sen Lin}
\affiliation{%
  \institution{Arizona State University}
  \city{Tempe}
  \country{USA}}
\email{slin70@asu.edu}

\author{ Mehmet Dedeoglu}
\affiliation{%
  \institution{Arizona State University}
  \city{Tempe}
  \country{USA}}
\email{mdedeogl@asu.edu}

\author{Junshan Zhang}
\affiliation{%
  \institution{Arizona State University}
  \city{Tempe}
  \country{USA}}
\email{junshan.zhang@asu.edu}

\begin{abstract}
   Online meta-learning is emerging as an  enabling technique for achieving edge intelligence in the IoT ecosystem. Nevertheless, to learn a good meta-model for within-task fast adaptation, a single agent alone has to learn over many tasks, and this is  the so-called `cold-start' problem. Observing that in a multi-agent network the learning tasks across different agents often share some model similarity, we ask the following fundamental question: ``Is it possible to accelerate the online meta-learning across agents via limited communication  and if yes how much benefit can be achieved? " To answer this question, we propose a multi-agent online meta-learning framework and cast it as an equivalent \emph{two-level nested online convex optimization (OCO)} problem.
    By characterizing the upper bound of the agent-task-averaged regret, we show that the performance of multi-agent online meta-learning depends heavily on how much an agent can benefit from the distributed network-level OCO for meta-model updates via limited communication, which however is not well understood. To tackle this challenge, we devise a distributed online gradient descent algorithm with \emph{gradient tracking} where each agent   tracks the global gradient using only one communication step with its neighbors per iteration, and it results in  an average regret $O(\sqrt{T/N})$ per agent, indicating that a factor of $\sqrt{1/N}$ speedup over the optimal single-agent regret $O(\sqrt{T})$ after $T$ iterations, where $N$ is the number of agents. Building on this sharp performance speedup, we next develop a multi-agent online meta-learning algorithm and show that it can achieve the optimal task-average regret  at a faster rate of $O(1/\sqrt{NT})$ via limited communication, compared to single-agent online meta-learning. Extensive experiments corroborate the theoretic results. 
\end{abstract}

\ccsdesc[500]{Networks~Network performance analysis}
\ccsdesc[500]{Computing methodologies~Online learning settings}
\ccsdesc[500]{Theory of computation~Online learning theory}
\ccsdesc[500]{Theory of computation~Multi-agent learning}

\keywords{multi-agent network, online meta-learning, distributed online convex optimization, gradient tracking}

\maketitle

\section{Introduction}

Meta-learning \cite{finn2017model,mishra2017simple,snell2017prototypical} has recently emerged as a promising approach for few-shot learning, aiming to solve new learning tasks quickly with only a few data samples by leveraging the prior knowledge from many related tasks. In particular, the gradient-based meta-learning \cite{finn2017model,nichol2018first} has become  popular because of its simplicity yet great effectiveness. Specifically, a meta-model is learnt across a set of training tasks sampled from some task distribution, such that the task-specific model for a new task can be quickly adapted from this meta-model via gradient descent using a few local samples. Such a fast learning capability with small datasets is critical for achieving artificial intelligence locally in resource-constrained devices,  paving the way to edge intelligence in the Internet-of-Things (IoT) ecosystem \cite{lin2020edge}. 

To enable continual lifelong learning as human beings do, much attention is being paid  to
online meta-learning \cite{denevi2019learning,finn2019online,khodak2019provable,khodak2019adaptive}, which
can be viewed as a synergy of two distinct learning methods, i.e., meta-learning and online learning \cite{shalev2011online}. Specifically, in online meta-learning, online learning tasks  arrive one at a time, and the agent intends to learn good priors based on its own experience about past tasks in a sequential manner so as to adapt quickly to the current task, and thus has a strong flavor of continual lifelong learning.  
Notably, \cite{khodak2019provable,khodak2019adaptive} study the gradient-based meta-learning algorithms in the framework of online convex optimization (OCO), where both within-task adaptation and update of the meta-models across tasks are treated as a OCO problem. 

Despite the superior fast learning performance of online meta-learning, to learn a good meta-model for within-task fast adaptation, a single agent alone still has to learn over many tasks, which inevitably encounters the cold-start problem. Observe that in a multi-agent network, the learning tasks across different agents in the same environment often share some model similarity \cite{smith2017federated}. For example, different robots may perform similar coordination behaviors according to the environment changes. In fact, one of the most remarkable abilities of human being is to continuously speed up learning of new tasks based on previous experiences from oneself as well as from others. Thus inspired, one may wonder if the cold-start problem for a single agent could be mitigated via limited collaboration among multiple agents by leveraging the task similarity therein. Here by ``limited collaboration" we mean limited communication between neighboring agents only, as the communication cost usually is a bottleneck in wireless communication systems. To be more specific, we seek to answer the following open questions: \emph{1) Can we accelerate the online meta-learning at a single agent on average in a multi-agent network, with only one communication step among neighbors per learning task? 2) If yes, how much can we improve upon the single-agent case?}

In this work, we give an affirmative answer to the first question, and show that the optimal task-average regret can be achieved at a faster rate for each agent in the multi-agent network via limited communication, compared to single-agent online meta-learning. More specifically, we propose MAOML, a multi-agent online meta-learning framework, which generalizes the single-agent online meta-learning framework, ARUBA, in \cite{khodak2019adaptive} to a multi-agent online meta-learning setting. In particular, we cast the multi-agent online meta-learning into an equivalent \emph{two-level nested} OCO problem, where we treat the within-task adaptation as a standard \emph{task-level} OCO problem, and the meta-model update as a distributed \emph{network-level} OCO problem across the multi-agent network. Mathematically, it can be shown that the performance ceiling of the multi-agent online meta-learning, in terms of the task-average regret, heavily depends on the performance of the distributed network-level OCO for the meta-model update. This is intuitive as a good meta-model should be able to capture the most important information across different tasks in the multi-agent network for enabling fast learning of a new task. \emph{Therefore, the problem of accelerating online meta-learning boils down to improving the performance per agent of distributed network-level OCO via limited communication.}

Then, the next key question is ``\emph{how much can an agent benefit from distributed OCO through limited communication with its neighbors?''} To  this end, consider a multi-agent network with $N$ agents. Intuitively, the more agents there are and the more information exchange, the smaller the average regret would be, and this is of interest particularly in a networked system. 
It is well known that the optimal regret in single-agent OCO is of order $O(\sqrt{T})$  after $T$ iterations, achievable by either online gradient descent (OGD) or follow-the-regularized-leader (FTRL) \cite{shalev2011online,hazan2019introduction}.
Interestingly, \cite{dekel2012optimal} and \cite{kamp2014communication} suggest that an average regret of $O(\sqrt{T/N})$, i.e., a factor of $\sqrt{1/N}$ speedup, can be obtained at each agent for multi-agent stochastic OCO, by performing  the synchronizations of local model predictions after each (or multiple) iteration. However,  the required synchronization  (for the model predictions) where all agents need to communicate until reaching consensus \cite{dekel2012optimal}  \cite{kamp2014communication}, incurs a significant communication burden, requiring $\Theta(\mathcal{Q}T)$ communication steps with $\mathcal{Q}$ being the diameter of the network, and hence inevitably suffers from  the latency which degrades the learning performance. \emph{In a nutshell, it remains unclear a priori if  distributed OCO algorithms  can achieve  significant improvement in terms of the average regret per agent, with only one communication step per iteration}.

The main contributions in this paper can be summarized as follows.
\begin{itemize}
    \item We propose a multi-agent online meta-learning framework to address the cold-start problem in single-agent online meta-learning, by leveraging the task similarity, i.e., the tasks follow some unknown distribution as in standard meta-learning \cite{finn2017model}, across multiple agents via limited communication. Along the line of the ARUBA framework introduced in \cite{khodak2019adaptive}, we treat the multi-agent online meta-learning as a two-level nested OCO problem, where the within-task adaptation and the meta-model update are formulated as a standard \emph{task-level} OCO problem and a distributed \emph{network-level} OCO problem across the multi-agent network, respectively.  
    
    \item We characterize the performance upper bound of multi-agent online meta-learning in terms of the \emph{agent-task-averaged regret}, and show that it heavily depends on how much an agent can benefit from the distributed network-level OCO for updating the meta-models through limited communication with its neighbors, which is  unclear a priori. To tackle this challenge, we further consider a distributed online gradient descent algorithm (DOGD-GT) with \emph{gradient tracking} \cite{qu2017harnessing,pu2020distributed}. We show that  by carefully tracking of the accumulated gradient consensus error through only limited communication among multiple agents,  the average  regret per agent can be significantly reduced to  $O(\sqrt{T/N})$ compared with the single-agent case,
    thus revealing a linear speedup of the learning performance.

    \item Building on the agent-level performance speedup benefiting from the multi-agent collaboration via gradient tracking in the distributed network-level OCO, we next propose a multi-agent online meta-learning algorithm called MAOML. It can be shown that each agent in MAOML can achieve a notable performance improvement in terms of the average regret per agent, i.e., approaching the optimal within-task regret at a faster rate of $O(1/\sqrt{NT})$ compared with the rate of $\Tilde{O}(1/\sqrt{T})$ in the single-agent online meta-learning ARUBA. 
    To the best of our knowledge, this is the first work to the address the cold-start problem by studying multi-agent online meta-learning under limited communication.
    
    \item We conduct extensive experiments on various datasets to demonstrate the performance of DOGD-GT and MAOML. The experimental results clearly indicate the improvement of MAOML over the single-agent online meta-learning in terms of the agent-task-averaged performance, corroborating the benefits of utilizing the task similarity across multiple agents through limited communication in both convex and nonconvex setups.

\end{itemize}

The rest of the paper is organized as follows. We present the related work in Section 2, and introduce the multi-agent online meta-learning framework in Section 3. In Section 4, we take a closer look to the distributed network-level OCO, and study the DOGD-GT algorithm. Building on the agent-level performance speedup achieved in the  distributed network-level OCO, we next propose a multi-agent online meta-learning algorithm MAOML in Section 5 with the performance analysis. The experimental study is presented in Section 6, followed by the conclusion in Section 7. 

\section{Related Work}

\emph{Online meta-learning.} Meta-learning has achieved great success in few-shot learning under the batch statistical setting  \cite{ravi2016optimization,finn2017model,nichol2018first}.
A gradient-based meta-learning algorithm called MAML is proposed in the seminal work \cite{finn2017model}, where a model initialization is learnt based on a lot of training tasks sampled from some task distribution, such that maximal performance at a new task can be achieved with the task-specific model quickly adapted from the model initialization via only one gradient descent step. To circumvent the need of Hessian computation in MAML, \cite{nichol2018first} studies a first-order meta-learning algorithm named Reptile.

Online meta-learning has recently received much attention. Particularly, \cite{finn2019online} extends the MAML algorithm \cite{finn2017model} to the online setting and proposes a follow-the-meta-leader algorithm. By applying the stochastic gradient descent to a proxy of true risk for a task based on a bias vector, \cite{denevi2019learning} proposes an online meta-algorithm by incrementally updating the bias when new tasks arrive, and quantifies the average excess risk bound. By building a decent connection between online meta-learning and OCO, \cite{khodak2019provable,khodak2019adaptive} study the gradient-based meta-learning algorithms in the framework of OCO. Moreover, \cite{denevi2019online} considers a general class of within-task learning based on primal-dual online learning, and \cite{zhuang2020no} extends the general online meta-learning  to the non-convex setting and evaluates the performance in terms of the local regret. 
In this paper, we make a first attempt to study the online meta-learning in a multi-agent scenario aiming to improve task-average performance.

\emph{Distributed OCO.} 
Distributed OCO \cite{yan2012distributed} in a multi-agent network has recently garnered much interest, where each agent first learns the model parameters  based on its local data and then communicates its local model information  with its neighbors. However, little attention has been paid to understand the impact of the network size on the average regret achievable at each individual agent therein.
To reap the potential benefits that an agent can achieve when carrying out distributed OCO, a convex loss function with both adversarial and stochastic components is considered in  \cite{zhao2019decentralized}. Assuming that the expected gradient is bounded above by $G$ and the stochastic variance is bounded above by $\sigma^2$, they have shown that the network expected regret is $O(\sqrt{N^2G^2T+NT\sigma^2})$. In contrast, \cite{dekel2012optimal} studies distributed OCO in a stochastic setup and proposes a distributed mini-batch algorithm, which leads to a network regret of order $O(\sqrt{NT})$, i.e., an agent-average regret of $O(\sqrt{T/N})$, indicating a possible linear speed-up of the average regret per  agent. However, there is a hidden cost associated with the needed synchronization  among all agents, required  at each iteration, 
which could incur a significant communication burden and learning performance degradation. To reduce the communication cost, a dynamic synchronization strategy is proposed in \cite{kamp2014communication} by reducing the frequency of synchronization, which however requires a central coordinator and still suffers from learning latency because of the  synchronization.



A  gradient-tracking based distributed OGD algorithm is considered in \cite{zhang2019distributed} for distributed OCO problem. However, the results therein are  different from ours, as outlined next: 1) \cite{zhang2019distributed} aims to show that the dynamic regret of distributed OCO has no explicit dependence on the time horizon, as in the centralized case, whereas we focus on characterizing the performance speedup by cleverly exploiting the limited multi-agent collaboration. 2) The results in \cite{zhang2019distributed}  rely on the assumption that the loss function  is strongly-convex, which is required even in the centralized case so as to remove the dependence on the time horizon. Since our focus is on the multi-agent speedup, strong-convexity is not a necessity and we  consider convex loss functions instead. 3) The dynamic regret defined in \cite{zhang2019distributed} cannot simply generalize to the problem setup in our setting, and a  non-trivial analysis of the regret bound is needed  to quantify the performance speedup.

\section{Multi-Agent Online Meta-Learning}

In this section, we first introduce the multi-agent online meta-learning framework, and cast it into an equivalent two-level nested OCO problem. By characterizing the upper bound of the agent-task-averaged regret, we show that the performance of the distributed network-level OCO for the meta-model update, is the bottleneck for the performance of multi-agent online meta-learning.


\subsection{Problem Formulation} 

As is standard in a multi-agent network, we assume that the agents communicate in an undirected and connected communication graph $\mathcal{G}=(\mathcal{V},\mathcal{E})$, where $\mathcal{V}\triangleq\mathcal{N}=\{1, ..., N\}$ is the set of vertices (agents) and $\mathcal{E}\subset \mathcal{V}\times \mathcal{V}$ is the set of edges connecting agents. Agent $i$ and $j$ can communicate with each other if and only if $(i,j)\in\mathcal{E}$. We further denote $\mathcal{N}_i=\{j|j\neq i, (i,j)\in\mathcal{E}\}$ as the set of neighbors of agent $i$. Each agent can  make its decision based on the local information and the information obtained from its neighbors via weighted averaging. To model this `weighting' process, a consensus weight matrix, $W=[w_{ij}]\in\mathbb{R}^{N\times N}$, is usually introduced with the following properties: 
\begin{itemize}
    \item For any $(i,j)\in\mathcal{E}$, we have $w_{ij}>0$; otherwise, $w_{ij}=0$. In particular, $w_{ii}>0$.
    \item Matrix $W$ is doubly stochastic, i.e., $\sum_{i'}w_{i'j}=\sum_{j'}w_{ij'}=1$ for all $i,j\in\mathcal{N}$.
\end{itemize}

In the multi-agent online meta-learning framework,
each agent $n\in\mathcal{N}$ faces with a sequence of online learning tasks $\mathcal{T}_{t,n}$ indexed by $t=1,...,T$, as illustrated in Figure \ref{Fig:maoml}. We assume that all agents are synchronized at the task level, i.e., new tasks arrive at all agents  at the same time. For each learning task $\mathcal{T}_{t,n}$, the agent $n$ must sequentially choose $m_{t,n}$ actions $\theta^i_{t,n}$ from some convex compact set $\Theta$ and incur loss $l^i_{t,n}:\Theta\rightarrow\mathbb{R}$ which is convex and Lipschitz, for $i\in[1,m_{t,n}]$. After learning one task, each agent would share learned model knowledge with its neighbors through one communication step to facilitate the learning of new tasks.

\begin{figure}
\centering
\includegraphics[scale=0.4]{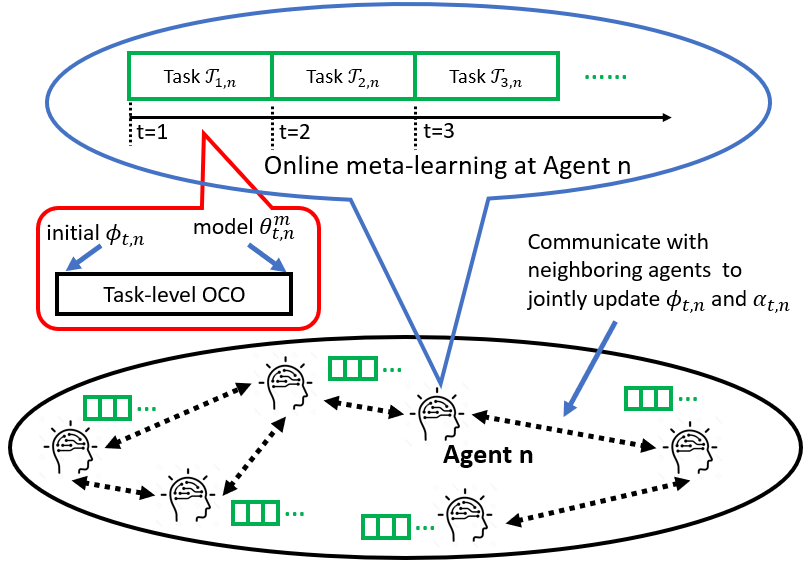}
\caption{The framework of multi-agent online meta-learning. Each agent in the multi-agent network has a sequence of online learning tasks, and it shares the learned model knowledge ($\phi_{t,n}$ and $\alpha_{t,n}$) with its neighbors to facilitate the learning of new tasks at time $t$.}
\label{Fig:maoml}
\end{figure}
Let $\theta^*_{t,n}$ denote the optimal model parameter for task $\mathcal{T}_{t,n}$, i.e., $\theta^*_{t,n}=\arg\min_{\theta\in\Theta}\sum_{i=1}^{m_{t,n}} l^i_{t,n}(\theta)$. Following the standard assumption in meta-learning \cite{finn2017model}, we assume that all the optimal model parameters $\theta^*_{t,n}$ for any $t\in[1,T]$ and $n\in\mathcal{N}$ follow some unknown distribution $\mathcal{P}_{\mathcal{T}}$, so as to capture the task similarity across the network.
In multi-agent online meta-learning, the agents aim to obtain good learning performance for each individual task. 
In the same spirit with \cite{khodak2019adaptive},  we study the  \textbf{agent-task-averaged regret} (ATAR) after each agent encounters $T$ tasks:
\begin{align*}
    \brn_a=\frac{1}{NT}\sum_{n=1}^N\sum_{t=1}^T\left(\sum_{i=1}^{m_{t,n}} l_{t,n}^i(\theta^i_{t,n})-\sum_{i=1}^{m_{t,n}} l^i_{t,n}(\theta^*_{t,n})\right).
\end{align*}
A low ATAR ensures that the individual task regret of an algorithm  is small on average over the network, compared to that of the optimal within-task parameter.
To this end, every agent in the network can
collaboratively learn, through limited communications with neighboring agents, the meta-models, i.e., \emph{a model initialization $\phi_{t,n}$ and a task-dedicated learning rate $\alpha_{t,n}$} by utilizing other agents' information, such that good within-task performance can be achieved with $\theta^i_{t,n}$ adapted from $\phi_{t,n}$ during the online meta-learning.

\subsection{Two-Level Nested OCO}

Based on the ARUBA framework \cite{khodak2019adaptive}, we treat the multi-agent online meta-learning 
as  a {\em two-level nested OCO} problem, and develop a theoretical framework for understanding the performance of multi-agent meta-learning through the lens of distributed OCO. For simplicity, we assume $m_{t,n}=m$, for any $t\in[1, T]$ and $n\in\mathcal{N}$.

\subsubsection{Task-level OCO} For the task $\mathcal{T}_{t,n}$ at the agent $n$, given the model initialization $\phi_{t,n}$ and within-task learning rate $\alpha_{t,n}$ learned jointly  based on the previous tasks, the agent seeks  to determine the action $\theta^i_{t,n}$ so as to minimize the within-task regret after $m$ rounds:
    \begin{align*}
        \brn_{t,n}=\sum\nolimits_{i=1}^m l_{t,n}^i(\theta^i_{t,n})-\sum\nolimits_{i=1}^m l^i_{t,n}(\theta^*_{t,n}).
    \end{align*}
 For a convex and $G$-Lipschitz loss function, it is well-known that the best upper bound for $\brn_{t,n}$ of online mirror descent (OMD),  regularized by Bregman divergence, is given as follows \cite{shalev2011online}:
    \begin{align}\label{within-upper}
        \brn_{t,n}\leq \frac{1}{\alpha_{t,n}}\mathcal{B}_R(\theta^*_{t,n}||\phi_{t,n})+\alpha_{t,n} G^2 m=\brh_{t,n},
    \end{align}
where for a  continuously-differentiable strictly convex function $g:\Theta\rightarrow \mathbb{R}$, the Bregman divergence is defined as
    \begin{align*}
        \mathcal{B}_R(\theta||\phi)=g(\theta)-g(\phi)-\langle \nabla g(\phi),\theta-\phi\rangle.
    \end{align*}
   This step corresponds to the within-task adaptation from the initial model $\phi_{t,n}$ using gradient descent regularized by the Bregman divergence, i.e., the inner loop of meta-learning.
   In order to use OCO for the meta-update of initial model $\phi_{t,n}$, we only consider the regularization as the set of Bregman divergence that is \emph{convex and smooth in the second argument}, i.e., $\mathcal{B}_R(\theta||\cdot)$ is convex and smooth for any fixed $\theta\in\Theta$. For example, when $g(\cdot)$ is the negative generalized entropy function defined for the expected loss of convex proper loss functions, the corresponding Bregman divergence satisfies the above condition \cite{painsky2019bregman}. The widely used $L_2$ regularization also satisfies this condition.

\subsubsection{Network-level OCO} Based on the definition of ATAR, it is clear that ATAR can be bounded above by the average of  $\{\brh_{t,n} \}$: 
    \begin{align}\label{upperATAR}
        \brn_a=\frac{1}{NT}\sum_{n=1}^N \sum_{t=1}^T \brn_{t,n}\leq \frac{1}{NT}\sum_{n=1}^N \sum_{t=1}^T\brh_{t,n}\triangleq\bbr_a,
    \end{align}
 which indicates that the ATAR is small if the average regret-upper-bound $\bbr_a$ is small. 
 Observe that each agent chooses one action pair ($\phi_{t,n}$, $\alpha_{t,n}$) and incurs the loss $\brh_{t,n}$ for each task $\mathcal{T}_{t,n}$. It follows that the outer loop of multi-agent online meta-learning, i.e., meta-update of the model initialization $\phi_{t,n}$ and the learning rate $\alpha_{t,n}$, can be cast as a distributed network-level  OCO among all $N$ agents.
 The objective here is to learn good meta-models ($\phi_{t,n}$, $\alpha_{t,n}$) for each agent via the multi-agent collaboration so as to minimize the following regret:
    \begin{align}\label{networkOCO}
        \brn_{out}=\frac{1}{NT}\sum_{n=1}^N\sum_{t=1}^T \left[\brh_{t,n}(\phi_{t,n},\alpha_{t,n})-\brh_{t,n}(\phi^*,\alpha^*)\right]
    \end{align}
where $(\phi^*,\alpha^*)=\arg\min \mathbb{E}_{\mathcal{P}_{\mathcal{T}}}[\brh_{t,n}(\phi,\alpha)]$. This distributed network-level  OCO enables the task-similarity to be learned on-the-fly, which is encapsulated in an adaptive learning rate by utilizing the information across the multi-agent network.

Note that the average regret-upper-bound $\bbr_a$ corresponds to the average loss in the distributed network-level  OCO for updating the meta-models. It is clear that $\bbr_a$ is small if the regret $ \brn_{out}$ is small for the distributed network-level  OCO, which consequently results in a small ATAR based on \eqref{upperATAR}. This is intuitive as the performance of online meta-learning directly depends on how good the meta-models are. In other words, if we could quickly learn good meta-models, i.e., the model initialization and learning rate, by utilizing the knowledge across the multi-agent network, good performance can be guaranteed for each task in online meta-learning, without the need of learning over many tasks at a single agent. Therefore, the problem of accelerating distributed online meta-learning  boils down to the problem of improving the performance per agent of distributed network-level  OCO, i.e., quickly learn good meta-models, via limited communication.

\section{Distributed Network-Level Online Convex Optimization}
\label{section3}

As alluded to earlier, it remains unclear a priori if any distributed OCO algorithms can achieve significant improvement in terms of the average regret per agent, with only one communication step per iteration. To tackle this challenge and also accelerate online meta-learning, we take a closer look to the distributed network-level OCO in this section, and devise a distributed OGD algorithm with gradient tracking.

For ease of exposition, we consider a more general formulation \cite{dekel2012optimal,hazan2012projection,chen2018projection,xie2020efficient} for the distributed network-level OCO \eqref{networkOCO}: 
In iteration $t$ the agent $i$ makes a local model prediction $x_{t,i}$ from a convex compact set $\mathcal{K}\subset\mathbb{R}^d$ and incurs convex loss $f_{t,i}(x_{t,i})$ that follows some unknown distribution $\mathcal{P}$, i.e., $f_{t,i}\sim \mathcal{P}$, for any $t$ and $i\in\mathcal{N}$. The stochastic assumption about the loss function corresponds to the underlying task distribution $\mathcal{P}_{\mathcal{T}}$ of meta-learning in an implicit manner.
The objective here is to make a sequence of  predictions $\{x_{t,i}\}$ given the knowledge of previous ones and possibly additional information so as to minimize the  average  regret (achieved at each agent) compared with the best predictor, given as:
\begin{align}\label{staticreg}
    \brn=\frac{1}{N}\left[\sum_{i=1}^N\sum_{t=1}^T f_{t,i}(x_{t,i})-\sum_{i=1}^N\sum_{t=1}^T f_{t,i}(x^*)\right],
\end{align}
where $x^*=\arg\min \mathbb{E}_{f_{t,i}\sim\mathcal{P}}[f_{t,i}(x)]$. 
Note that the above problem formulation is closely related to but different from the classical stochastic optimization  in the following sense \cite{dekel2012optimal}: Stochastic optimization is primarily concerned with finding the optimal solution efficiently,  for a given underlying model distribution.
In stark contrast,   for the (stochastic) online convex optimization, each agent  makes a sequence of decisions in a real-time manner   when new data arrives, and the objective is to make a sequence of model predictions that results in a small cumulative loss along the way. In this study, distributed OCO algorithms are devised to  reduce the average regret per agent with limited communication,  compared with the single agent case.

Since the regret depends on the distribution of   $f_{t,i}$, we focus on  the expected regret $\mathbb{E}[\brn]$, which is the same across agents because $\{f_{t,i}\}$ follow the same unknown distribution $\mathcal{P}$.  It is well known that in the centralized case  OGD can achieve the optimal regret $\mathbb{E}[\brn]=O(\sqrt{T/N})$ after totally $NT$ iterations are executed sequentially. In the distributed case where each agent runs OGD alone with no communication, it is clear that the regret $\mathbb{E}[\brn]$ at each agent has the order of $O(\sqrt{T})$, which is a factor of $\sqrt{N}$ worse than the centralized case. This performance gap points to the need of the collaboration among agents in order to obtain the optimal regret per agent.

\subsection{Distributed OGD with Gradient Tracking}


Gradient tracking has shown great potentials in distributed optimization to improve the convergence rate through the collaboration among agents  \cite{qu2017harnessing,pu2020distributed,tang2019distributed,li2020communication}. Particularly, by taking advantage of the smoothness of the local functions, an accurate estimation of the global gradient can be obtained as a better descent direction based on the history information, in contrast to gradient descent with local gradients. Nevertheless, the benefit of gradient tracking, especially the acceleration capability, is not well understood in distributed online learning where one cares about the learning process. To fully unleash the potential of gradient tracking, 
we explore a distributed OGD algorithm with gradient tracking (DOGD-GT) in order to achieve the performance speedup at each agent for distributed OCO, as outlined in Algorithm \ref{ogd_track}.

More specifically, an auxiliary variable $s_{t,i}$ is introduced for each agent to track the average gradients over the network by leveraging history information:
\begin{align*}
    s_{t,i}=\sum_{j\in\mathcal{N}_i} w_{ij}s_{t-1,j}+\nabla f_{t,i}(x_{t,i})-\nabla f_{t-1,i}(x_{t-1,i}),
\end{align*}
which serves as a more accurate estimation of the global gradient $\frac{1}{N}\sum_i\nabla f_{t,i}(x_{t,i})$, in contrast to the local gradient $\nabla f_{t,i}(x_{t,i})$.
As a result, the local model at each agent is updated based on $s_{t,i}$ using the gradient descent:
\begin{align*}
    x_{t+1,i}=\sum_{j\in\mathcal{N}_i} w_{ij}x_{t,j}-\eta s_{t,i}.
\end{align*}
Compared with the standard distributed OGD (DOGD) algorithms, DOGD-GT has the same order of the communication cost,  which is much smaller than that in the distributed mini-batch algorithm proposed in \cite{dekel2012optimal}, where additional consensus steps are needed in the network after every iteration.

\begin{algorithm}[H]
	\caption{Distributed OGD with gradient tracking}
	\label{ogd_track}
 	\begin{algorithmic}[1]
		  \State Initialize $x_{1,i}=0$  for all $i\in\mathcal{N}$;
		  \For{$t= 1, 2, ..., T$}
		    \For{each agent $i$}
		        \State Apply local model $x_{t,i}$ and incur loss $f_{t,i}(x_{t,i})$;
		        \State Compute gradient $\nabla f_{t,i}(x_{t,i})$;
		        \If{$t=1$}
		        \State Query the local model $x_{t,j}$ from  neighbors $j\in\mathcal{N}_i$; 
		        \State Compute $s_{t,i}=\nabla f_{t,i}(x_{t,i})$;
		        \Else
		        \State Query the local model $x_{t,j}$ and $s_{t-1,j}$ from neighbors $j\in\mathcal{N}_i$;
		        \State Update $s_{t,i}=\sum_{j\in\mathcal{N}_i} w_{ij}s_{t-1,j}+\nabla f_{t,i}(x_{t,i})-\nabla f_{t-1,i}(x_{t-1,i})$;
		        \EndIf
		        \State Update  $x_{t+1,i}=\sum_{j\in\mathcal{N}_i} w_{ij}x_{t,j}-\eta s_{t,i}$;
		    \EndFor
		\EndFor
	\end{algorithmic}
\end{algorithm}

\subsection{Performance Analysis}

We next quantify the performance speedup brought by the limited collaboration among agents in DOGD-GT. We first impose the following standard assumptions.


\begin{assumption}
Each $f_{t,i}(x)$ is convex and $L$-smooth. And there exists some constant $D$ such that $\mathbb{E}[\|\nabla f_{t,i}(x)\|^2]\leq D$.
\end{assumption}

\begin{assumption}
Let $F(x)=\mathbb{E}[f_{t,i}(x)]$. The stochastic gradient $\nabla f_{t,i}(x)$ has a $\sigma^2$-bounded variance, i.e., there exists a constant $\sigma\geq0$ such that
\begin{align*}
    \mathbb{E}[\|\nabla f_{t,i}(x)-\nabla F(x)\|^2]\leq \sigma^2.
\end{align*}
\end{assumption}

Let $\rho$ denote the spectral norm of $W-\frac{1}{N}\bone\bone^T$ where $\bone$ denotes an $N$-dimensional all one column vector, then $\rho\in (0,1)$. Moreover, it can be shown that \cite{qu2017harnessing}
\begin{align}
    \|Wx-\bone \Bar{x}\|\leq \rho \|x-\bone \Bar{x}\|
\end{align}
where $\Bar{x}=\frac{1}{N}\bone^T x$.

Let $\Bar{x}_t=\frac{1}{N}\sum_{i=1}^N x_{t,i}$, and $x_t=[x'_{t,1},x'_{t,2},\cdots,x'_{t,N}]'$.
To analyze the regret of DOGD-GT, 
we note that the techniques in stochastic optimization \cite{pu2020distributed} cannot be directly applied here, because  it is necessary to track the regret accumulated within the learning process instead of the optimality gap $\lim_{t\rightarrow \infty} (f_{t,i}(x_{t,i})-f_{t,i}(x^*))$.
In light of this, we decompose the regret into two parts: (a) the  regret $\sum_{i=1}^N\sum_{t=1}^T[f_{t,i}(x_{t,i})-f_{t,i}(\Bar{x}_t)]$ resulted from the consensus error among agents, and (b)  the regret $\sum_{i=1}^N\sum_{t=1}^T[f_{t,i}(\Bar{x}_t)-f_{t,i}(x^*)]$ accumulated over the iterations of $\Bar{x}_t$. 

For (a), we first have the following lemma to characterize the relationship between the regret and the consensus gap between model parameters.
\begin{restatable}{lemma}{lemmarestate}\label{lemma3}
Under Assumption 1,  the following inequality holds:
\begin{align*}
    \mathbb{E}\left[\sum_{i=1}^N\sum_{t=1}^T f_{t,i}(x_{t,i})-\sum_{i=1}^N\sum_{t=1}^T f_{t,i}(\Bar{x}_t)\right]\leq 2L\sum_{t=1}^T \mathbb{E}\left[\|x_t-\bone \Bar{x}_t\|^2\right].
\end{align*}
\end{restatable}
The proofs for all Lemmas and Theorems in this work can be found in the appendix.
Next, we follow a similar way as in \cite{pu2020distributed} to build a linear system to bound the consensus error $E[\|x_{t,i}-\bar{x}_t\|^2]$. 
\begin{restatable}{lemma}{lemmarestatea}\label{lemma_3}
Let $\alpha=\frac{3+\rho^2}{4}$. Under Assumptions 1 and 2, the following inequality holds for some constant $A_1$ and $A_2$:
\begin{align*}
    &\sum_{t=1}^T \mathbb{E}[\|x_t-\bone \Bar{x}_t\|^2]
    \leq A_1\frac{\alpha-\alpha^T}{1-\alpha}+\|x_1-\bone\Bar{x}_1\|^2\\
    &+A_2\eta^2\frac{1+\rho^2}{1-\rho^2}[18\eta^2\sigma^2L^2+18N\eta^2L^2D+(1+\eta LN+N)\sigma^2]T.
\end{align*}
\end{restatable}
The challenge lies in the characterization of the convergence rate of the consensus error, which needs a careful manipulation and analysis of the coefficient matrices in the linear system. 

For (b), the key question is how to analyze this regret term without strong convexity. The techniques from \cite{qu2017harnessing} and \cite{pu2020distributed} cannot be applied, as the former considers that each agent has the same loss function in the entire learning process and the later assumes the strong convexity. To resolve this issue, we quantify both the optimaltiy gap at iteration $t+1$, i.e., $f_{t,i}(\bar{x}_{t+1})-f_{t,i}(x^*)$, and the one-iteration gap between iteration $t$ and iteration $t+1$, i.e., $f_{t,i}(\bar{x}_t)-f_{t,i}(\bar{x}_{t+1})$. In this way, we can characterize the relationship between the optimality gap and the consensus error, and bound the one-iteration gap by the norm of global gradients, which leads to the following result. 
\begin{restatable}{lemma}{lemmarestateb}\label{lemma7}
Under Assumptions 1 and 2, the following inequality holds:
\begin{align*}
&\mathbb{E}\left[\sum_{i=1}^N\sum_{t=1}^T f_{t,i}(\Bar{x}_t)-\sum_{i=1}^N\sum_{t=1}^T f_{t,i}(x^*)\right]\\
    \leq& \frac{4N\|\Bar{x}_1-x^*\|^2}{\eta}+26L\sum_{t=1}^T\mathbb{E}[\|x_t-\bone\Bar{x}_t\|^2]+\frac{N\eta}{2}\mathbb{E}[\|\nabla F(\Bar{x}_{T+1})\|^2]\\
    &+2\sigma^2\eta T+24L\mathbb{E}[\|x_{T+1}-\bone\Bar{x}_{T+1}\|^2].
    \end{align*}
\end{restatable}

Based on Lemma \ref{lemma3}, \ref{lemma_3} and \ref{lemma7}, 
we have the following result about the average regret per agent.

\begin{theorem}\label{theorem1}
Under Assumptions 1 and 2, when $\eta$ satisfies that
\begin{align*}
    \eta\leq \min\left\{\frac{(1-\rho^2)^{1.5}}{32L\sqrt{1+\rho^2}},\frac{1}{2L}\sqrt{\frac{N}{T}}\right\}
\end{align*}
with $N=o(T^{1/3})$, the DOGD-GT algorithm attains the following regret bound:
\begin{align*}
    \mathbb{E}[\brn] 
    =& O\left(\eta^2T+\frac{\eta T}{N}+\frac{1}{\eta}\right)
    = O\left(\sqrt{\frac{T}{N}}\right).
\end{align*}
\end{theorem}

\begin{remark}
(1) 
Theorem \ref{theorem1} indicates that  each agent can achieve a factor of $\sqrt{1/N}$ speedup in terms of the average  regret $\mathbb{E}[\brn]$, through only one communication step per iteration by leveraging gradient tracking, compared to the case where a single agent can  achieve a regret of order $O(\sqrt{T})$ without collaboration with other agents. 
(2) The overall regret obtained by DOGD-GT, i.e., $O(\sqrt{NT})$, also matches the optimal regret in the centralized case where $NT$ iterations are processed sequentially. 
(3) Note that the learning rate $\eta$ requires the knowledge of the time horizon $T$, which however can be relaxed by applying a standard doubling trick \cite{cesa2006prediction}.
\end{remark}
\begin{remark}
The classical DOGD algorithm \cite{zhao2019decentralized} cannot achieve such  performance gain in the  setting here, because essentially DOGD performs a consensus step followed by a gradient descent along the local gradient $\nabla f_{t,i}(x_{t,i})$. For a fixed learning rate, DOGD only converges to a neighborhood of the optimizer $x^*$, because the local gradient is data-driven and hence random. Such an oscillation around $x^*$ slows down the convergence and results in a larger regret, calling for a more elegant consensus algorithm. This is also corroborated by the consensus schemes in the work on distributed multi-armed bandits \cite{landgren2016distributed,shahrampour2017multi}.
In contrast, gradient tracking provides an efficient way to communicate local estimations of the global gradient with the neighbors, and each agent is able to quickly construct a more accurate estimate of the global gradient $\frac{1}{N}\sum_i\nabla f_{t,i}(x_{t,i})$ with only one communication step per iteration as the information diffuses in the network until consensus. And the global gradient estimation clearly serves as a better direction than the local gradient no matter the stochasticity is in place or not, leading to a better regret bound.
\end{remark}

\section{MAOML}

Thanks to gradient tracking, the proposed DOGD-GT algorithm clearly showcases the potential for accelerating the learning process in distributed OCO through limited collaboration among agents. 
To reap the potential benefits, we next devise a multi-agent online meta-learning (MAOML) algorithm based on DOGD-GT, to mitigate the cold-start problem.

    \begin{algorithm}[H]
	\caption{MAOML}
	\label{maoml}
 	\begin{algorithmic}[1]
		  \State Initialize $\phi_{1,n}$ and $\alpha_{1,n}$ for all $n\in\mathcal{N}$;
		  \For{$t= 1, 2, ..., T$}
		    \For{each agent $n$}
		        \State Receive task $\mathcal{T}_{t,n}$ which would be learnt for $m$ rounds; 
		        \For{round $i\in[m]$}
		            \State Run online mirror descent with $\phi_{t,n}$ and $\alpha_{t,n}=\frac{v_{t,n}}{G\sqrt{m}}$ to obtain $\theta^i_{t,n}$; ~~//~~ \textbf{(within-task adaptation)}
		            \State Incur loss $l^i_{t,n}(\theta^i_{t,n})$;
		        \EndFor
		        \State Run DOGD-GT with all agents to update $\phi_{t+1,n}$ and $\alpha_{t+1,n}=\frac{v_{t+1,n}}{G\sqrt{m}}$;
		        ~~//~~ \textbf{(multi-agent meta-update of OMD initialization and learning rate)}
		    \EndFor
		\EndFor
	\end{algorithmic}
\end{algorithm}

As shown in \eqref{within-upper},  $\brh_{t,n}$ is a joint function for $\phi_{t,n}$ and $\alpha_{t,n}$, and it would be  easier to learn $\phi_{t,n}$ and $\alpha_{t,n}$ separately  \cite{khodak2019adaptive}. Specifically, the distributed network-level OCO can be decoupled as two separate distributed OCOs over the following two function sequences $\{f^{init}_{t,n}\}_{t,n}$ and $\{f^{rate}_{t,n}\}_{t,n}$ for every task at each agent:
\begin{align*}
    f^{init}_{t,n}(\phi)&=\mathcal{B}_R(\theta^*_{t,n}||\phi)G\sqrt{m},\\
    f^{rate}_{t,n}(v)&=\left(\frac{\mathcal{B}_R(\theta^*_{t,n}||\phi_{t,n})}{v}+v\right)G\sqrt{m}.
\end{align*}
In what follows, we make a few further remarks on the algorithm design:
\begin{itemize}
    \item Here for each agent $n$ at every task $\mathcal{T}_{t,n}$, the model initialization $\phi_{t,n}$ is updated based on DOGD-GT over the function $f^{init}_{t,n}$ for $\phi_{t,n}\in\Theta$, and the learning rate $\alpha_{t,n}=\frac{v_{t,n}}{G\sqrt{m}}$ where $v_{t,n}$ is updated based on DODG-GT over the function $f^{rate}_{t,n}$ for $v_{t,n}\geq \epsilon>0$. By assuming that  $\mathcal{B}_R(\theta||\phi)\leq H^2$ for any $\theta$, $\phi\in\Theta$, it is easy to check that $f^{rate}_{t,n}(v)$ is convex and $\frac{2H^2}{\epsilon^3}$-smooth for $v\in [\epsilon,\infty)$. 
    \item Note that although the global optimal $\phi^*$ and $\alpha^*$ exist for all tasks,
at each iteration $t$ different agents would have distinct model initialization $\phi_{t,n}$ and learning rate $\alpha_{t,n}$ for their current tasks. 
\item And for implementation, one can use the last iterate $\theta_{t,n}^m$ to replace the optimal $\theta^*_{t,n}$, which  incurs an additional $o(\sqrt{m})$ regret term only for many practical settings \cite{khodak2019adaptive}.
\end{itemize}
The details are summarized in Algorithm \ref{maoml}.

\subsection{Performance Analysis}

Based  on Theorem \ref{theorem1}, we  have the following result about the performance of MAOML.

\begin{theorem}\label{theorem2}
Suppose that  the model initialization $\phi_{t,n}$ and $v_{t,n}$ are updated based on DOGD-GT with
 $\alpha_{t,n}=\frac{v_{t,n}}{G\sqrt{m}}$. Then, the ATAR achieved by each agent in the multi-agent online meta-learning satisfies that
\begin{align*}
    \mathbb{E}[\brn_a]\leq \mathbb{E}[\bbr_a]= O\left(\frac{1+\frac{1}{V_{\phi}}}{\sqrt{NT}}+V_{\phi}\right)\sqrt{m}
\end{align*}
where $V_{\phi}^2=\min_{\phi\in\Theta}\mathbb{E}[\mathcal{B}_R(\theta^*_{t,n}||\phi)]$.
\end{theorem}

To obtain a more concrete sense  about the performance improvement of MAOML, we compare it with the single-agent online meta-learning, i.e., $N=1$ (thus ignore the subscript $n$). In particular, we apply the general algorithm (Algorithm 1 therein)  \cite{khodak2019adaptive} to our setting here, which yields the following proposition.

\begin{proposition}\label{corollary2}
Suppose that the model initialization is updated based on $\phi_{t}=\frac{1}{t-1}\sum_{j=1}^{t-1} \theta^*_j$, and the learning rate $\alpha_t=\frac{v_t}{G\sqrt{m}}$ where $v_t$ is updated using simplified exponentially-weighted online-optimization (EWOO) \cite{hazan2007logarithmic} with parameter $\epsilon=\frac{1}{T^{1/4}}$. Then, the ATAR achieved by the single-agent online meta-learning satisfies that
\begin{align*}
    \mathbb{E}[\brn_a]\leq\mathbb{E}[\bbr_a]= \Tilde{O}\left(\min\left\{\frac{1+\frac{1}{V_{\phi}}}{\sqrt{T}},\frac{1}{T^{1/4}}\right\}+V_{\phi}\right)\sqrt{m}
\end{align*}
where $V_{\phi}^2=\min_{\phi\in\Theta}\mathbb{E}[\mathcal{B}_R(\theta^*_{t}||\phi)]$.
\end{proposition}

\begin{remark}
(1) It can be seen from Proposition \ref{corollary2} that for the single-agent online meta-learning, if $V_{\phi}$, the average deviation of $\theta^*_{t,n}$, is $\Omega_T(1)$, then the   ATAR approaches $O(V_{\phi}\sqrt{m})$ at rate $\Tilde{O}(1/\sqrt{T})$. In contrast, with the same number $T$ of online learning tasks, each agent in multi-agent online meta-learning can achieve a clear performance gain by utilizing the task similarity across multiple agents through the limited collaboration, i.e., the ATAR approaches $O(V_{\phi}\sqrt{m})$ at a faster rate of $O(1/\sqrt{NT})$.  Although we consider $m_{t,n}=m$ for simplicity, it is worth to note that the results still hold as long as all $\{m_{t,n}\}$ follow some distribution across all tasks.

(2) Moreover, the result shown in Theorem \ref{theorem2} also matches the optimal performance in the centralized case with $NT$ tasks in total. It is worth to note that, for the set $\Theta$ with diameter $H$, the single-task regret achieved by OGD is $O(H\sqrt{m})$, whereas in online meta-learning the optimal regret for each task is  smaller, i.e., $O(V_{\phi}\sqrt{m})$ when the optimal $\theta_{t,n}^*$ are close, especially for the few-shot setting of a small $m$ \cite{khodak2019adaptive}. 

(3) Built on joint learning of the model initialization and the learning rate from all past tasks, online meta-learning is intimately related to the regularization-based methods, particularly the prior-focused methods, in continual learning \cite{de2019continual}. This strong connection indicates that the multi-agent online meta-learning methods can be used to speed up learning in continual learning, in particular, few-shot continual learning where each task only has a few data samples.
\end{remark}

\section{Experiments}

In what follows, we present extensive experiments on both DOGD-GT and MAOML which corroborate the theoretic results in previous sections, respectively.

\begin{figure*}
    \centering
    \subfigure[Performance comparison in stochastic setup for $N=8$.]{\label{fig:a}\includegraphics[width=43.5mm]{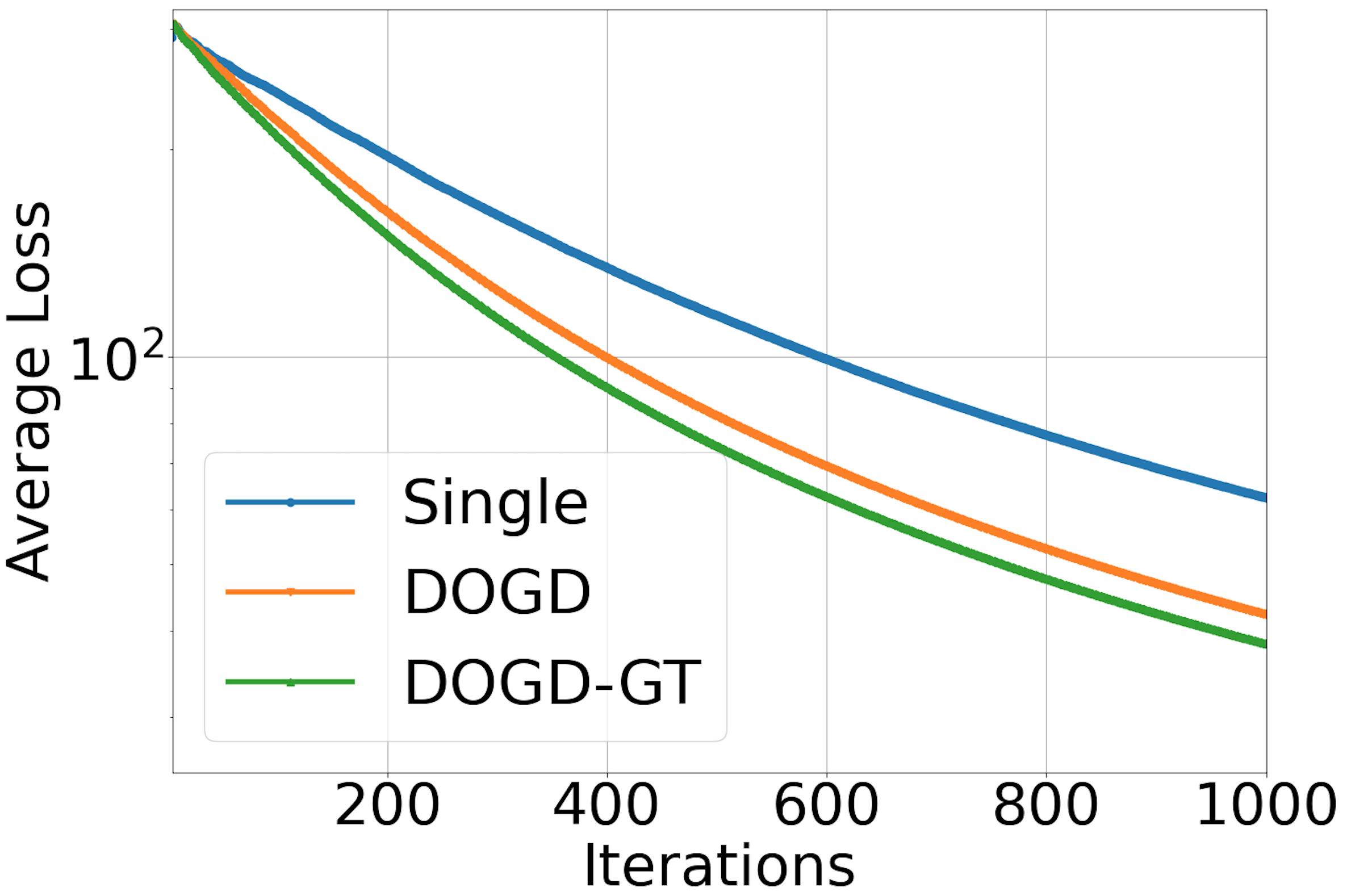}}
    \subfigure[Impact of $N$ in stochastic setup on DOGD-GT. ]{\label{fig:b}\includegraphics[width=43.5mm]{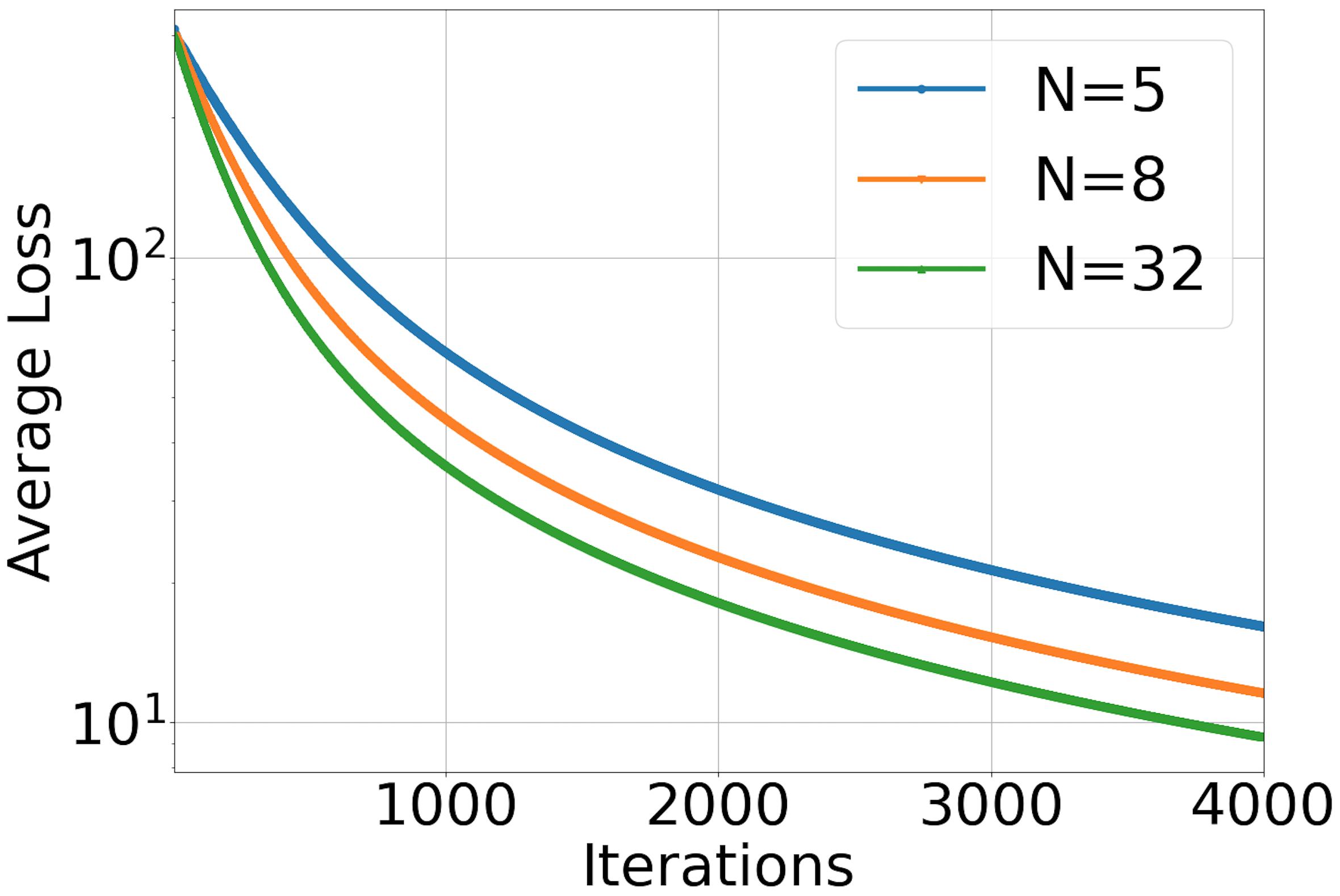}}
    \subfigure[Performance comparison in adversarial setup for $N=12$.]{\label{fig:c}\includegraphics[width=43.5mm]{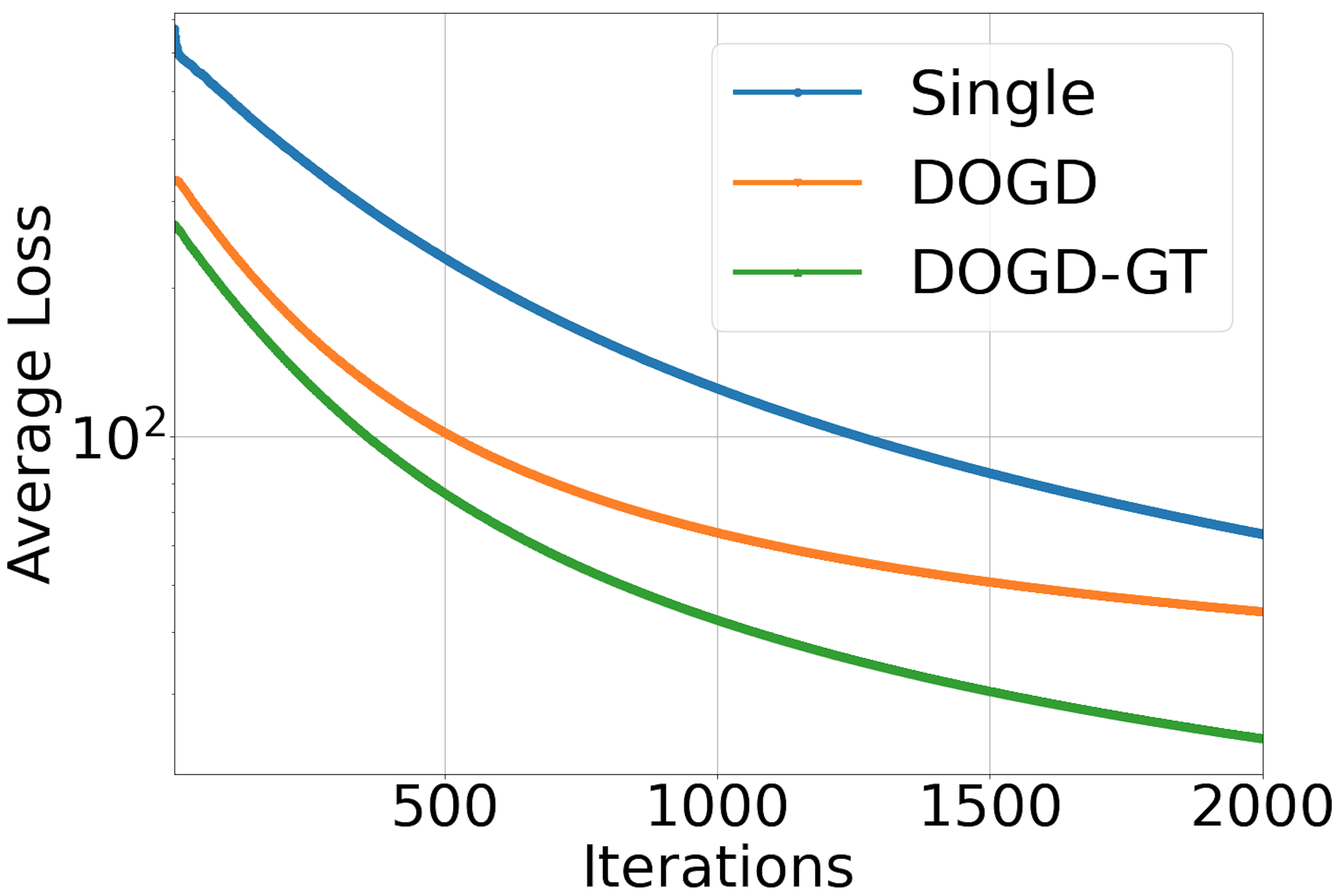}}
    \subfigure[Impact of $N$ in adversarial setup on DOGD-GT.]{\label{fig:d}\includegraphics[width=43.5mm]{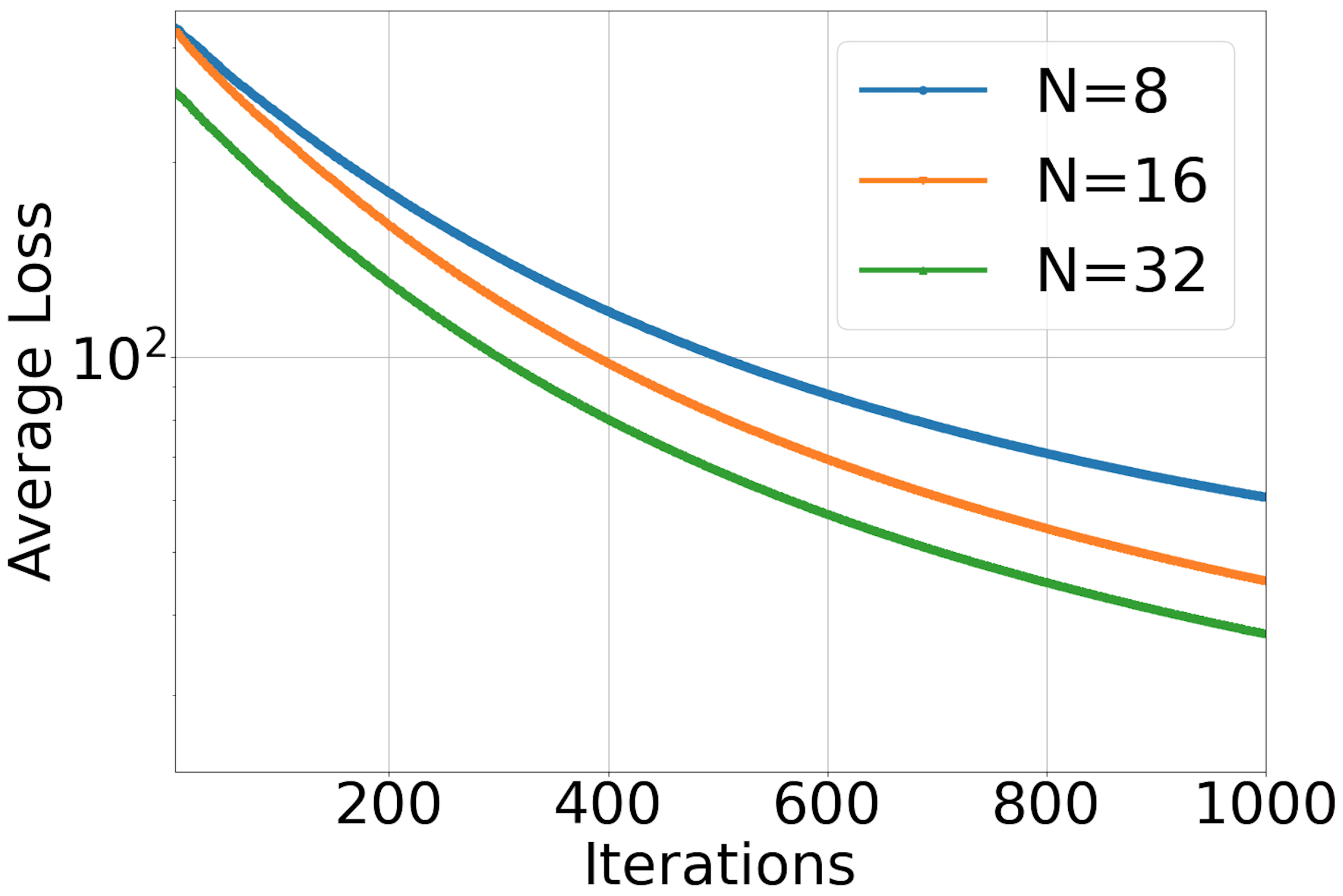}}
    \caption{Performance evaluations of DOGD-GT on synthesic data.}
    \label{fig:impact_conv}
\end{figure*}

\begin{figure*}
    \centering
    \subfigure[Performance comparison for 5-way 10-shot MNIST.]{\label{fig:ma}\includegraphics[width=44mm]{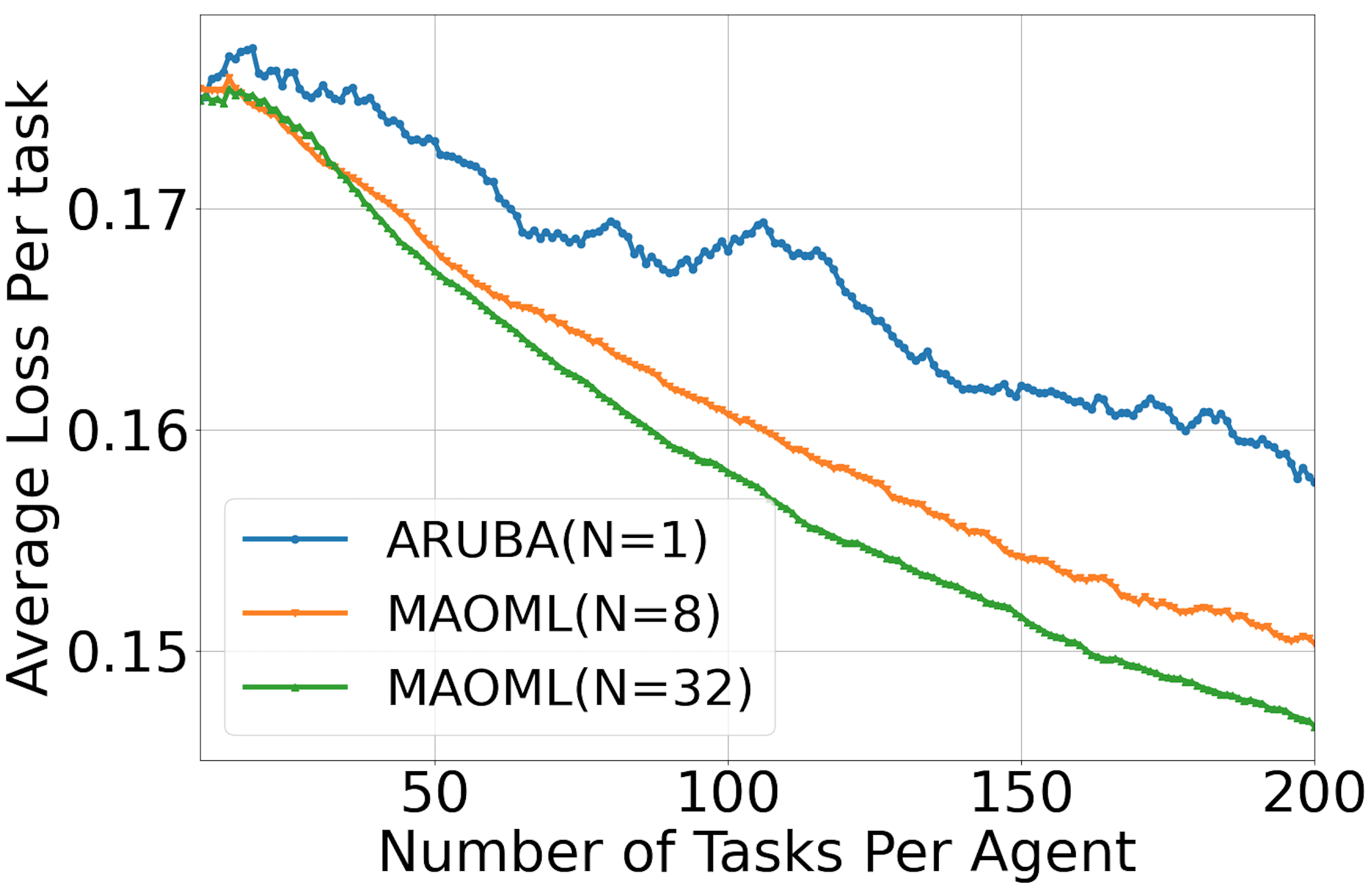}}
    \subfigure[Impact of $m$ on MAOML for 5-way 10-shot MNIST. ]{\label{fig:mb}\includegraphics[width=44mm]{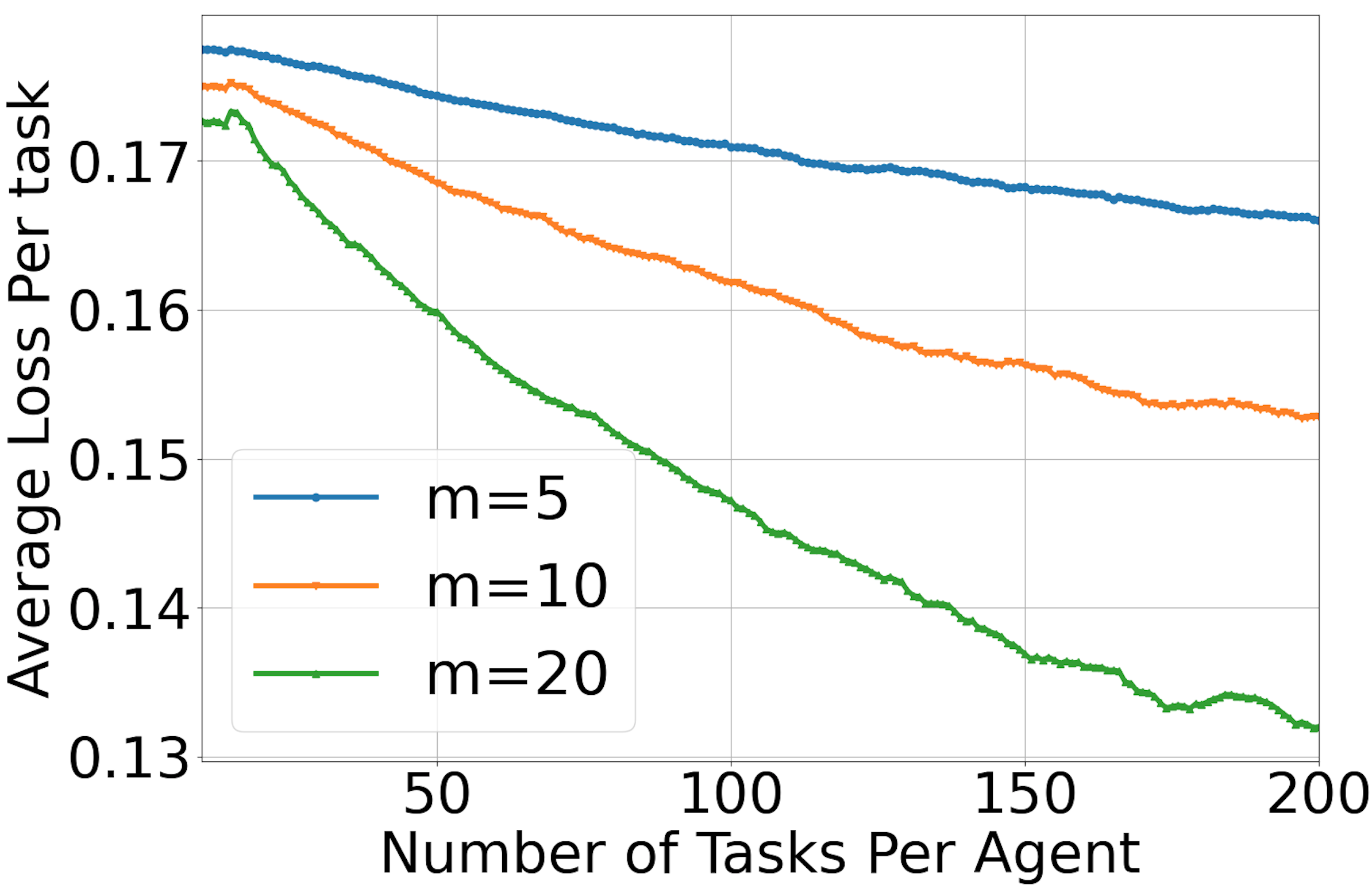}}
    \subfigure[Performance comparison for 5-way 5-shot Omniglot.]{\label{fig:mc}\includegraphics[width=43.1mm]{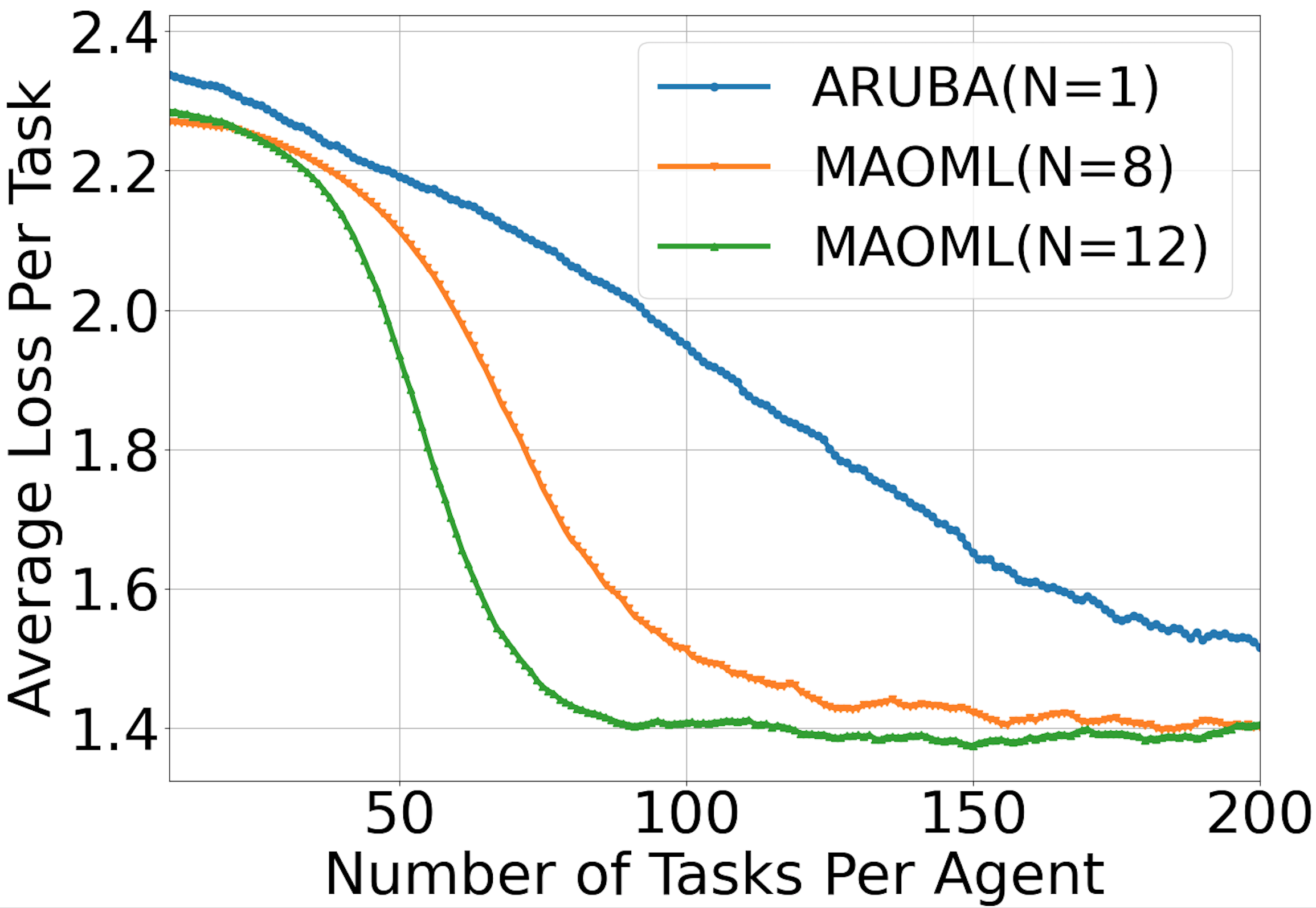}}
    \subfigure[Impact of $m$ on MAOML for 5-way 5-shot Omniglot. ]{\label{fig:md}\includegraphics[width=43.4mm]{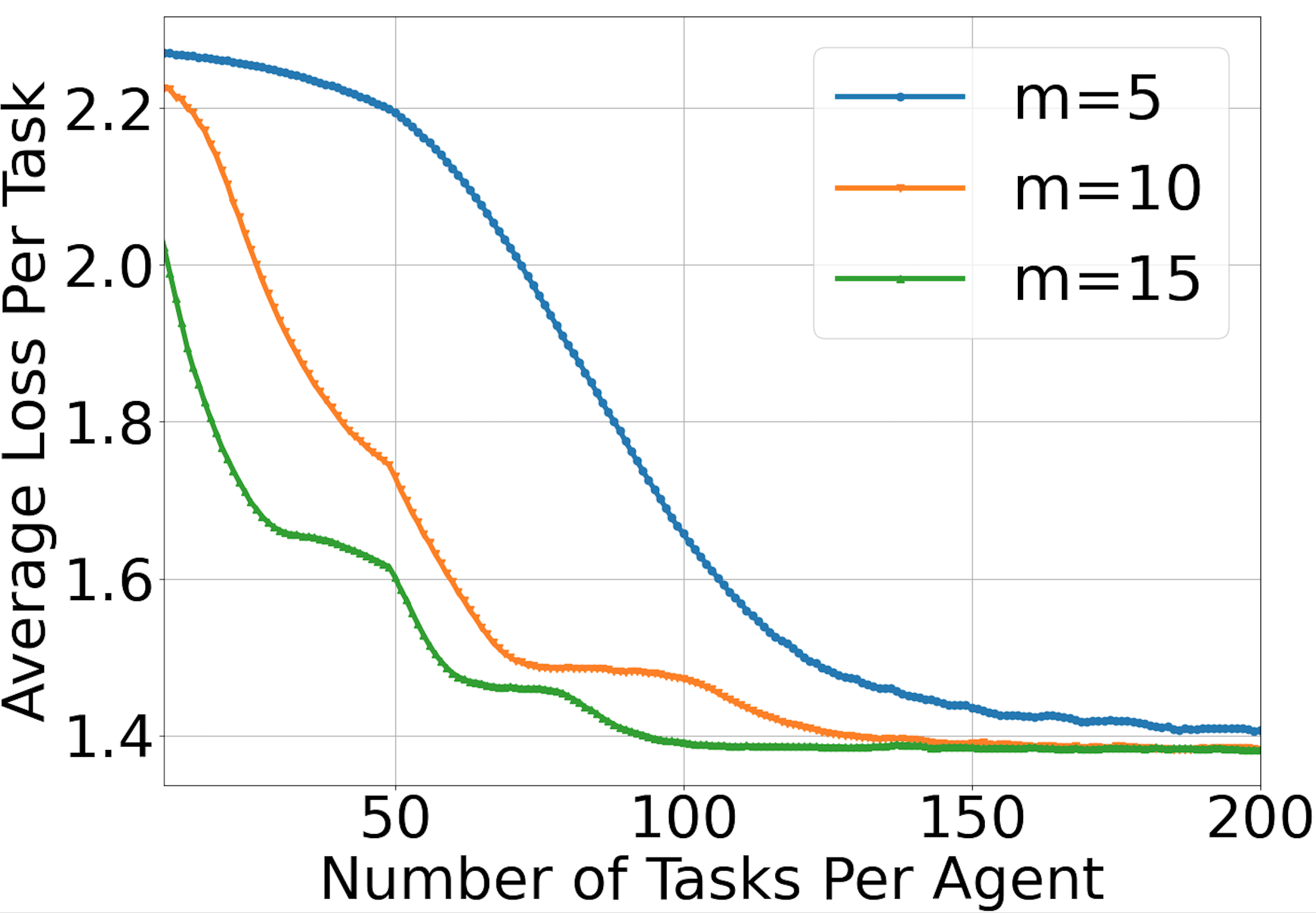}}
    \caption{Performance evaluations of MAOML on MNIST and Omniglot.}
    \label{fig:maoml}
\end{figure*}

We first introduce the setup of the communication graph for the multi-agent network. More specifically, we consider that $N$ agents communicate in a random network \cite{pu2020distributed,xie2020efficient}, where each two agents are linked with probability $0.5$ (discard the graphs that are not connected). And the weight matrix $W$ is defined based on the Metropolis rule \cite{sayed2014adaptive}:
\begin{align*}
    w_{ij}=\begin{cases}
    1/\max\{\deg(i),\deg(j)\} &\text{if}~j\in\mathcal{N}_i,\\
    1-\sum_{j\in\mathcal{N}_i}w_{ij} &\text{if}~i=j,\\
    0 &\text{otherwise},
    \end{cases}
\end{align*}
where $\deg(i)$ is the degree of agent $i$. We also consider a complete communication graph where all agents are connected with each other for evaluating the performance of MAOML.

\subsection{Performance of DOGD-GT}

As in \cite{qu2017harnessing, pu2020distributed}, we study the online Ridge regression problem, where each agent $i$ at each iteration $t$  incurs the following loss:
\begin{align*}
    f_{t,i}(x_{t,i})=(u_{t,i}^Tx_{t,i}-v_{t,i})^2+\rho\|x_{t,i}\|^2
\end{align*}
for a given model $x_{t,i}$ and the data sample $(u_{t,i},v_{t,i})$. Here $\rho>0$ is a penalty parameter.

In the experiments, each $u_{t,i}$ is uniformly sampled from $[0.3,0.4]^p$ with dimension $p$, and $v_{t,i}$ is generated according to $v_{t,i}=u_{t,i}^T\Tilde{x}_{t,i}+\epsilon_{t,i}$, where $\Tilde{x}_{t,i}$ is a predefined parameter, and $\epsilon_{t,i}$ are independent Gaussian random noises with mean $0$ and variance $0.5$. For completeness, 
we evaluate the performance of DOGD-GT in both stochastic and adversarial setups:
(1) \emph{Stochastic setup}: all $\Tilde{x}_{t,i}$ are the same in this case, set as a constant from $[0,5]^p$;
(2) \emph{Adversarial setup}:  $\Tilde{x}_{t,i}$ are randomly and independently located in $[0,10]^p$ in this case.
Moreover, $\rho=0.001$, $p=10$, and the learning rate $\eta=0.001$. We evaluate the average learning performance by measuring the average loss $\frac{1}{NT}\sum_{i=1}^N\sum_{t=1}^T f_{t,i}(x_{t,i})$ as in \cite{zhao2019decentralized} over multiple simulations.

To demonstrate the performance gain achieved by gradient tracking, we compare the performance of DOGD-GT with both DOGD and the single agent approach. 
Clearly, as shown in Figure \ref{fig:a} and \ref{fig:c}, DOGD-GT outperforms DOGD and the single agent approach in both stochastic and adversarial setups, indicating the benefits brought by collaborating with neighbors to track the global gradient via limited communication.  We also evaluate the impact of the network size $N$ on the performance of DOGD-GT. As expected, it can be seen from Figure \ref{fig:b} and \ref{fig:d} that the learning performance improves with $N$ for both stochastic and adversarial setups, validating the results in Theorem \ref{theorem1}. 

 To further validate the performance of DOGD-GT, we  study the online multiclass logistic regression on the MNIST dataset. For a batch $\mathcal{B}_{t,i}$ of data samples $j=(u_{t,i}^j,v_{t,i}^j)\in \mathbb{R}^d\times \{0,...,9\}$ where $u_{t,i}^j$ is the feature  and $v_{t,i}^j$ is the label, the logistic loss function for $x_{t,i}\in\mathbb{R}^{d\times 10}$ is defined as:
    \begin{align*}
        f_{t,i}(x_{t,i})=\frac{-1}{|\mathcal{B}_{t,i}|}\sum_{j\in\mathcal{B}_{t,i}}\sum_{v=0}^9 \bone\{v_{t,i}^j=v\}\log\frac{\exp(x_{t,i}^T u_{t,i}^j)}{\sum_{k=0}^9 \exp(x_{t,i}^T(k)u_{t,i}^j)}.
    \end{align*}
Specifically, each batch $\mathcal{B}_{t,i}$ is randomly and independently sampled from the entire MNIST dataset. We set $d=784$ and $\eta=0.0001$. As shown in Figure \ref{fig:amnist}, DOGD-GT outperforms both DOGD and the single agent case. When $N$ increases, the performance of DOGD-GT will be better, as illustrated in Figure \ref{fig:bmnist}.

\begin{figure}
    \centering
    \subfigure[Performance comparison for $N=8$.]{\label{fig:amnist}\includegraphics[width=85mm]{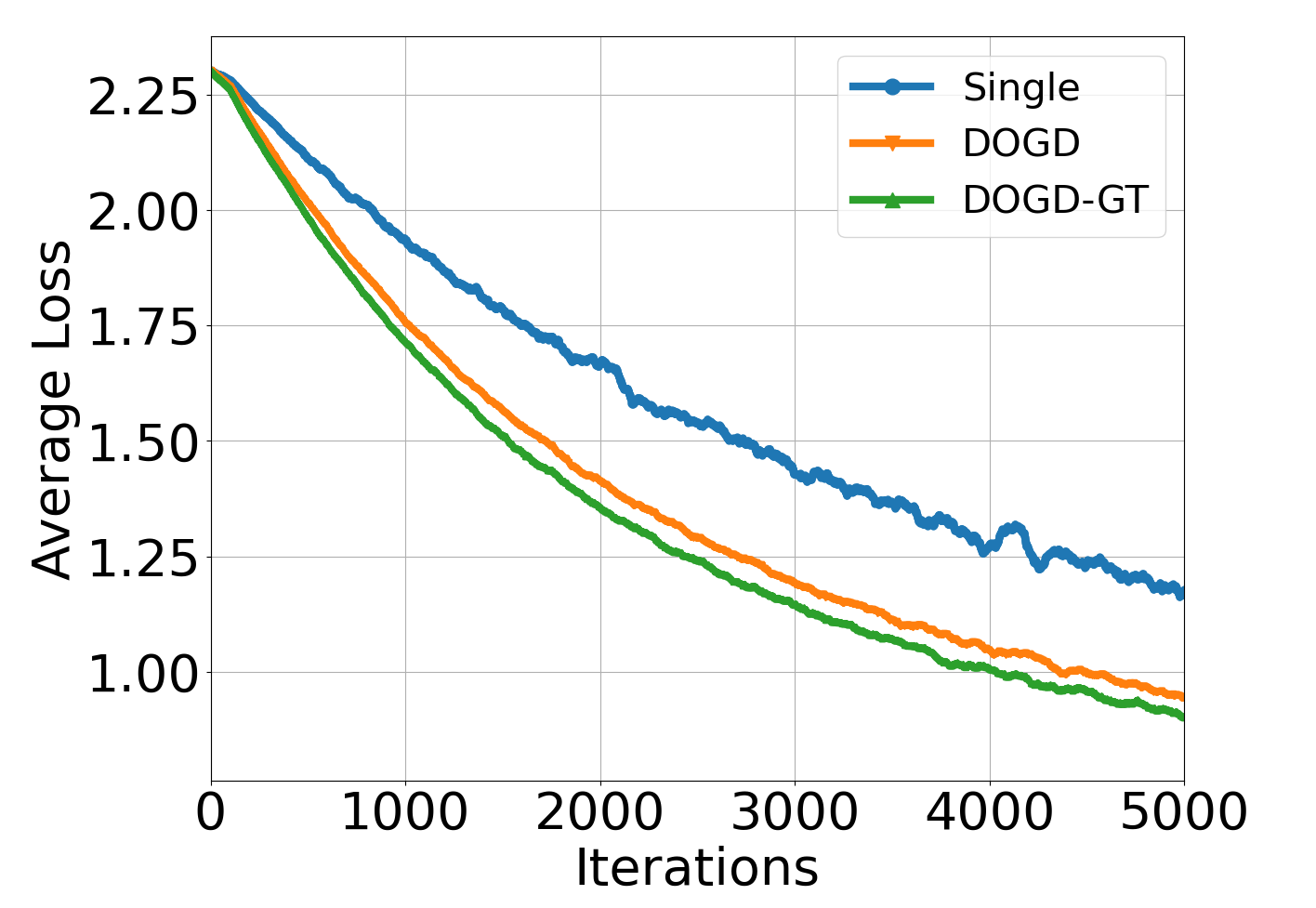}}
    
    \subfigure[Impact of $N$ on DOGD-GT. ]{\label{fig:bmnist}\includegraphics[width=85mm]{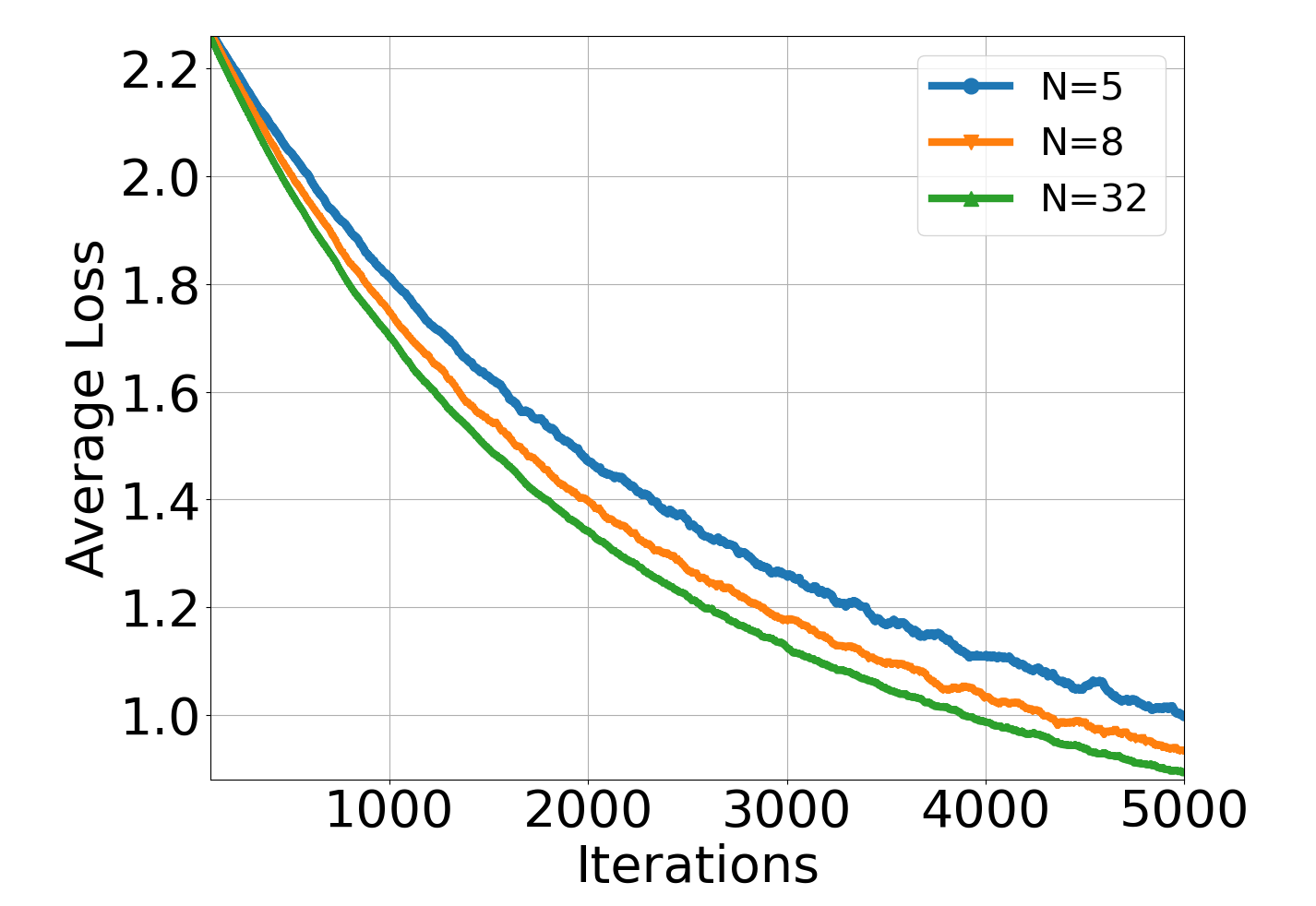}}
    \caption{Performance evaluation of DOGD-GT on MNIST.}
    \label{fig:maoml_mnist}
\end{figure}

\subsection{Performance of MAOML}

Next, we evaluate the performance of MAOML in both convex and nonconvex setups: (1) \emph{Convex setup}: we consider the online multiclass logistic regression \cite{xie2020efficient} on the MNIST dataset as an online learning task $\mathcal{T}_{t,n}$ at each agent. For a batch $\mathcal{B}^i_{t,n}$ of data points $j=(u_{t,n}^j,v_{t,n}^j)\in \mathbb{R}^d\times \{0,...,9\}$ where $u_{t,n}^j$ is the feature  and $v_{t,n}^j$ is the label, the logistic loss function for $\theta\in\mathbb{R}^{d\times 10}$ is defined as:
    \begin{align*}
        l_{t,n}^i(\theta)=\frac{-1}{|\mathcal{B}^i_{t,n}|}\sum_{j\in\mathcal{B}^i_{t,n}}\sum_{v=0}^9 \bone_{\{v_{t,n}^j=v\}}\log\frac{\exp(\theta^T u_{t,n}^j)}{\sum_{k=0}^9 \exp(\theta^T(k)u_{t,n}^j)}.
    \end{align*}
And for each agent we consider 5-way 10-shot classification with the dataset randomly sampled from the entire dataset.
 (2) \emph{Nonconvex setup}: we study 5-way 5-shot classification on Omniglot \cite{lake2011one} as the online learning task $\mathcal{T}_{t,n}$ using a deep neural network (DNN). The DNN architecture for each task consists of two \(2D\) convolutional layers (first with \(6\) output channels and second with \(16\) output channels) with kernel sizes \(5\times 5\). Each convolution operation is followed first by \(\mathrm{ReLu}\) non-linearity, and then by \(2D\) max-pooling operation with stride of \(2\). The final layer is a fully connected layer with input of size \(16\times 4\times 4\) and output of size \(10\). We deploy the cross entropy to quantify the loss with respect to a single sample.
In the experiments, we evaluate the average learning performance by measuring the average loss $\frac{1}{NTm}\sum_{n=1}^N\sum_{t=1}^T\sum_{i=1}^m l_{t,n}^i(\theta_{t,n}^i)$. Along the same line in \cite{khodak2019provable}, we use OGD as the learning algorithm within each task.

When applying DOGD-GT to update the model initialization $\phi_{t,n}$ and  $v_{t,n}$, in the experiments, we set the learning rate $\eta=0.001$ for the outer loop meta-update with DOGD-GT. For the selection of $G$, we test different values and choose the one with the best performance for every experimental setup. For example, when we use the logistic regression for the few-shot classification on  MNIST, we set $G=80$. We further clarify the parameters used in different experiments: (1) For Figure \ref{fig:ma}, we set $m=10$; (2) For Figure \ref{fig:mb} and \ref{fig:md}, we set $N=8$; (3) For Figure \ref{fig:mc}, we set $m=5$.

We first compare the performance of MAOML under different number of agents with the single-agent general algorithm ARUBA in \cite{khodak2019adaptive}. As shown in Figure \ref{fig:ma} and \ref{fig:mc}, MAOML clearly outperforms ARUBA, by utilizing the task similarity across multiple agents through limited communication in both convex and nonconvex setups. More specifically, compared with ARUBA, MAOML learns  good model priors at a faster rate, and performs significantly better  after each agent learns over the same number of tasks.
Moreover, with more agents collaborating in the network, the performance of MAOML increases further, corroborating the results in Theorem \ref{theorem2}. We next examine the impact of $m$, i.e., the number of iterations within each task, on the learning performance of MAOML. As expected, the average loss per task decreases with $m$ because each task has a sublinear regret on average, as illustrated in Figure \ref{fig:mb} and \ref{fig:md}.

Following the same line as in \cite{khodak2019provable,khodak2019adaptive}, we also evaluate the performance of MAOML in a meta-testing setup. More specifically, for each pair of $(\phi_{t,n},v_{t,n})$ obtained at each iteration $t$, we test its performance on a set of testing tasks. For each testing task $\mathcal{T}^{te}_{t,n}$ at each agent with a training dataset $\mathcal{D}^{tr}_{t,n}$ and a testing dataset $\mathcal{D}^{te}_{t,n}$, we first run online gradient descent from the model initial $\phi_{t,n}$ with the learning rate $\alpha_{t,n}=\frac{v_{t,n}}{G\sqrt{m}}$ for $m$ iterations using the training dataset $\mathcal{D}^{tr}_{t,n}$, and obtain the task specific model parameter $\theta^{te}_{t,n}$. Next, we evaluate the accuracy of $\theta^{te}_{t,n}$ on the testing dataset $\mathcal{D}^{te}_{t,n}$ for each testing task $\mathcal{T}^{te}_{t,n}$. 

We run the experiments for 5-way 2-shot classification and 5-way 5-shot classification on Omniglot, and evaluate the average testing accuracy over 10 testing tasks after each iteration $t$ for every agent. Particularly, we consider a complete graph where all agents are connected with each other, and set $m=50$.
As shown in Figure \ref{test:omniglot2} and \ref{test:omniglot5}, MAOML clearly achieves a better meta-testing accuracy compared with ARUBA, and its performance further increases as the number of agents $N$ increases. Therefore,  by utilizing the 
task similarity across different agents through limited collaboration among them, each agent can achieve good testing performance in MAOML after learning over a smaller number of tasks, in contrast to learning alone by itself.

\begin{figure}
    \centering
    \includegraphics[width=80mm]{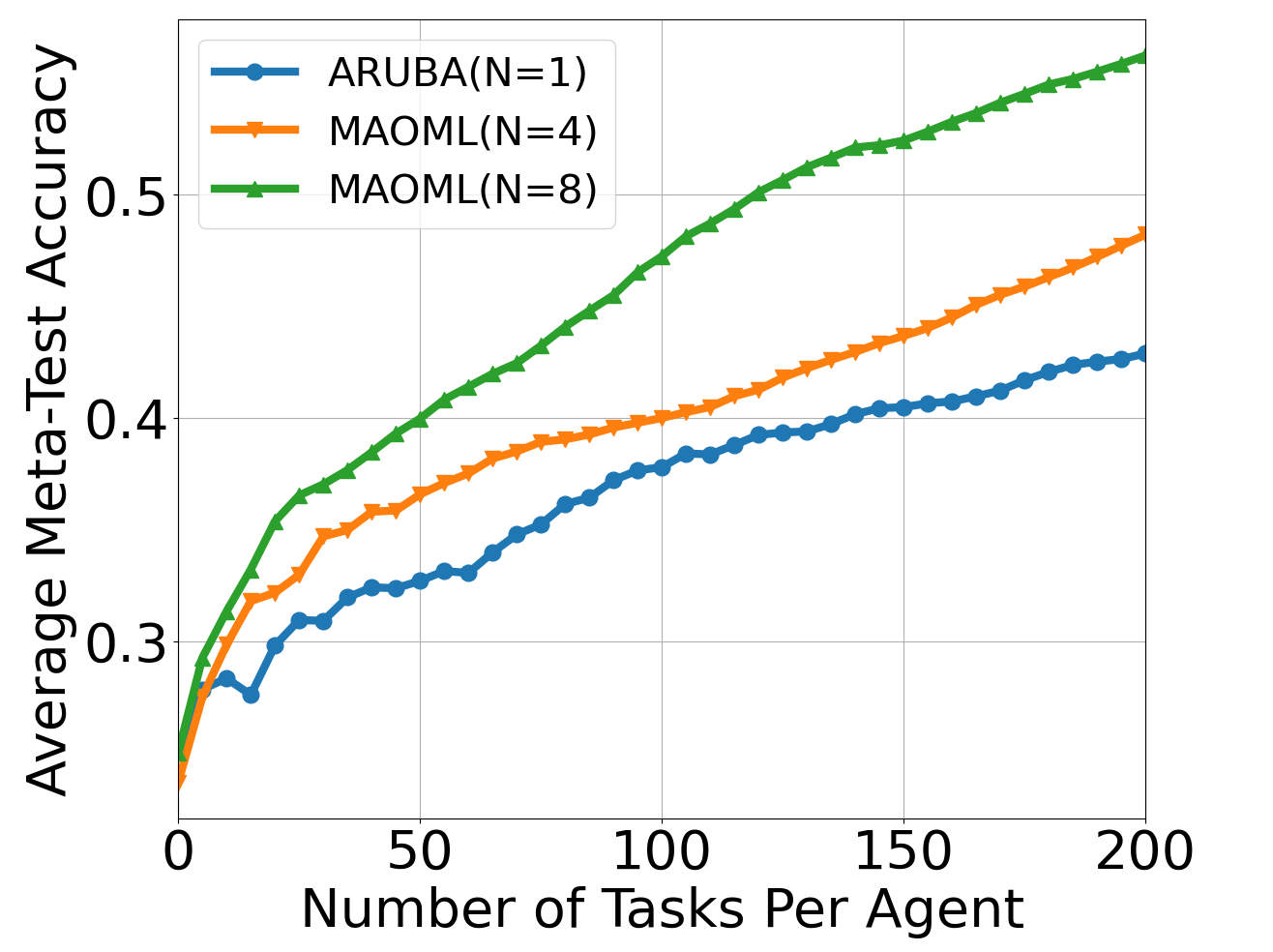}
    \caption{Meta-testing performance evaluation of MAOML on 5-way 2-shot Omniglot.}
    \label{test:omniglot2}
\end{figure}

\begin{figure}
    \centering
    \includegraphics[width=80mm]{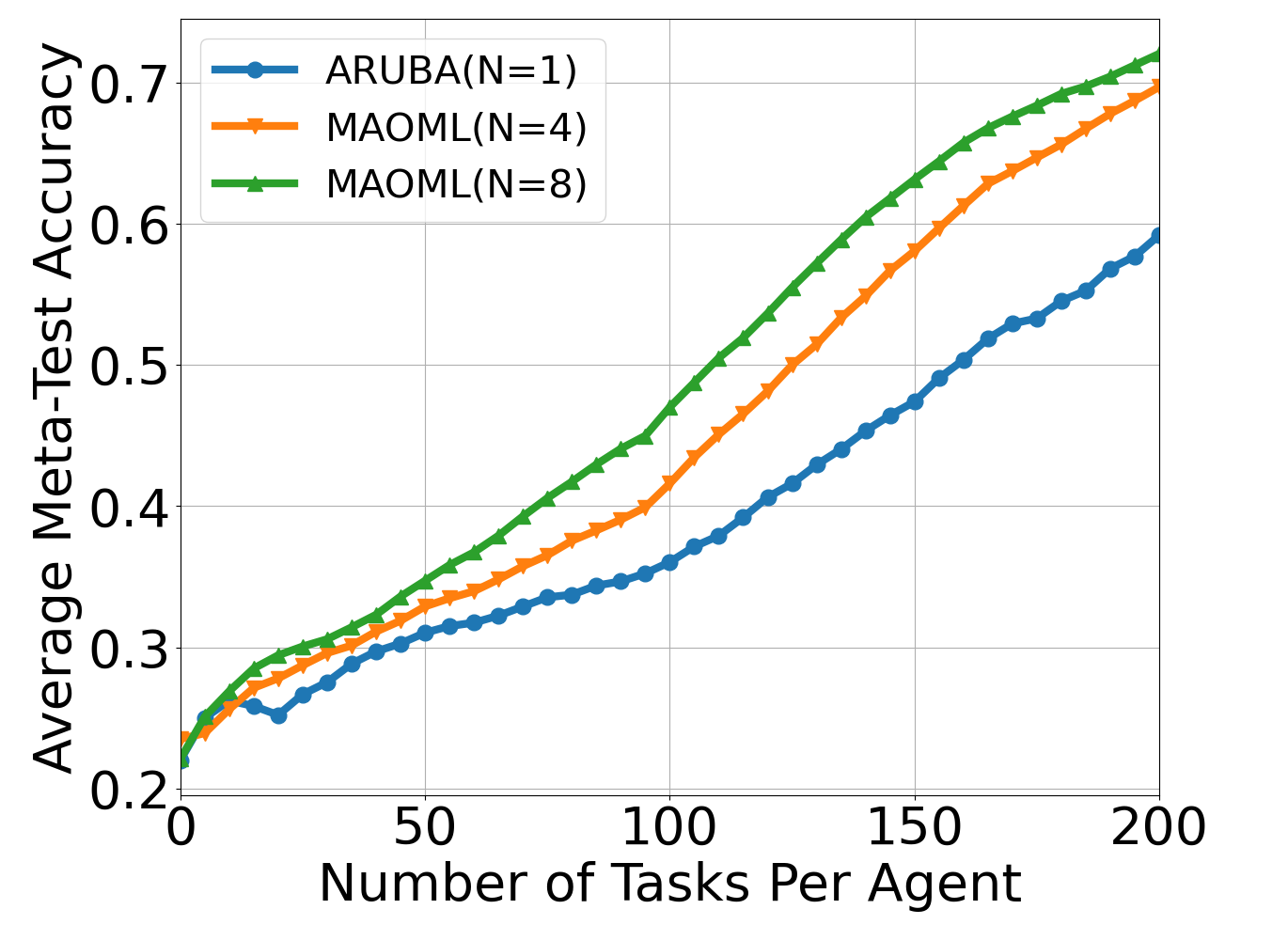}
    \caption{Meta-testing performance evaluation of MAOML on 5-way 5-shot Omniglot.}
    \label{test:omniglot5}
\end{figure}

\section{Conclusion}

In single-agent online meta-learning, the agent has to learn over many tasks so as to obtain good meta-models, based on which within-task fast adaptation can be achieved. Nevertheless, this would inevitably lead to the cold-start problem. To address this problem, we propose a multi-agent online meta-learning framework to leverage the task similarity across multiple agents, and cast it into an equivalent two-level nested OCO problem. By pinpointing that the performance bottleneck lies in the distributed network-level OCO, where it still remains unclear that how much an agent can benefit from it through limited communication with neighboring agents, we further explore a DOGD algorithm with gradient tracking. 
We show that the  average regret $O(\sqrt{T/N})$ can be achieved at each agent, thus revealing a linear speedup of the learning performance compared with the single-agent case. Building on the foundation of the agent-level performance speedup achieved  in the distributed network-level OCO, we next propose a multi-agent online meta-learning algorithm MAOML, and show that the optimal within-task regret can be achieved at a faster rate of $O(1/\sqrt{NT})$ compared with the rate of $\Tilde{O}(1/\sqrt{T})$ in the single agent case. The theoretic results have been clearly verified in the experimental studies on different datasets.

\bibliographystyle{ACM-Reference-Format}
\bibliography{sample-base}

\onecolumn
 \section*{Appendix}
\appendix

For ease of exposition, we define the following average sequences:
\begin{align*}
    \Bar{x}_t=\frac{1}{N}\sum_{i=1}^N x_{t,i},~~ \Bar{s}_t=\frac{1}{N}\sum_{i=1}^N s_{t,i},~~ g_t=\frac{1}{N}\sum_{i=1}^N \nabla f_{t,i}(x_{t,i}),
\end{align*}
and further, rewrite
$\{x_{t,i}\}$, $\{s_{t,i}\}$ and $\{\nabla f_{t,i}(x_{t,i})\}$ in vector form, i.e.,
\begin{align*}
    x_t=\begin{bmatrix}x_{t,1}\\x_{t,2}\\\vdots\\x_{t,N}\end{bmatrix},~~
    s_t=\begin{bmatrix}s_{t,1}\\s_{t,2}\\\vdots\\s_{t,n}\end{bmatrix},~~
    \nabla_t=\begin{bmatrix}\nabla f_{t,1}(x_{t,1})\\ \nabla f_{t,2}(x_{t,2})\\\vdots\\\nabla f_{t,n}(x_{t,n})\end{bmatrix}.
\end{align*}
Based on Algorithm \ref{ogd_track}, the update rule can be reformulated as 
\begin{align}
    s_t&=W s_{t-1}+\nabla_{t}-\nabla_{t-1},\\
    x_{t+1}&=W x_t-\eta s_t,
\end{align}
where $s_1=\nabla_1$.
We define $G_t=\begin{bmatrix}
\nabla F(x_{t,1})\\
\vdots\\
\nabla F(x_{t,N})
\end{bmatrix}$ as the expected gradient at $x_t$.

\section{Preliminaries}
To facilitate the regret analysis, we first restate some useful results in the literature. More specifically, to understand the updates of the average sequences $\Bar{s}_t$ and $\Bar{x}_t$, based on Lemma 7 in \cite{qu2017harnessing},  we have
\begin{lemma}\label{lemma1}
The following equalities hold.
\begin{enumerate}[label=(\alph*)]
    \item  \ \ \ $\Bar{s}_{t+1}=\Bar{s}_t+g_{t+1}-g_t=g_{t+1}$;
    \item \ \ \  $\Bar{x}_{t+1}=\Bar{x}_t-\eta\Bar{s}_t=\Bar{x}_t-\eta g_t$.
\end{enumerate}
\end{lemma}
\begin{proof}
Since $W$ is doubly stochastic, it follows that $\bone^T W=\bone^T$.
To prove (a), we have 
\begin{align*}
    \Bar{s}_{t+1}=&\frac{1}{N}\bone^T s_{t+1}\\
              =&\frac{1}{N}\bone^T(Ws_t+\nabla_{t+1}-\nabla_t)\\
              =& \Bar{s}_t+g_{t+1}-g_t.
\end{align*}
Telescoping the above equation, we have $\Bar{s}_{t+1}=\Bar{s}_1+g_{t+1}-g_1$. Since $\Bar{s}_1=g_1$, we can obtain $\Bar{s}_{t+1}=g_{t+1}$.

To prove (b), we have
\begin{align*}
    \Bar{x}_{t+1}=&\frac{1}{N}\bone^T x_{t+1}\\
                =&\frac{1}{N}\bone^T Wx_t-\eta s_t\\
                =&  \Bar{x}_t-\eta \Bar{s}_t\\
                =& \Bar{x}_t-\eta g_t.
\end{align*}
\end{proof}

Denote $\mathcal{F}_t$ as the $\sigma$-algebra generated by the sequence $\{f_1,f_2,...,f_{t-1}\}$ where $f_t=[f_{t,1}^T,...,f_{t,N}^T]^T$, and define $\mathbb{E}[\cdot | \mathcal{F}_t]$ as the conditional expectation given $\mathcal{F}_t$. Based on\cite{pu2020distributed}, we  have the following two lemmas.

\begin{lemma}\label{lemma2}
The following inequality holds:
\begin{align*}
    \mathbb{E}[\langle Ws_t-\bone g_t, G_{t}-\nabla_t\rangle|\mathcal{F}_t]\leq \sigma^2.
\end{align*}
\end{lemma}
\begin{lemma}\label{lemma_22}
The following inequality holds:
\begin{align*}
    \mathbb{E}[\langle G_{t+1}, G_t-\nabla_t\rangle|\mathcal{F}_t]\leq \eta LN\sigma^2.
\end{align*}
\end{lemma}

We also need the following standard\cite{chen1993convergence,tseng2008accelerated,dekel2012optimal}.
\begin{lemma}\label{lemma_new}
Let $V$  be a closed convex set, $\phi$ be a convex function on $V$, and $h$ be a differentiable, strongly convex function on $V$. Let $d$ be the Bregman divergence generated by $h$. Given $u\in V$, if
\begin{align*}
    w^+=\arg\min_{w\in V} \{\phi(w)+d(w,u)\},
\end{align*}
then
\begin{align*}
    \phi(w)+d(w,u)\geq \phi(w^+)+d(w^+,u)+d(w,w^+).
\end{align*}
\end{lemma}

\section{Regret Analysis}


In order to prove Theorem \ref{theorem1}, based on \eqref{staticreg}, we can first rewrite the regret as
\begin{align}\label{regretdecomp}
    N\brn=\brn_s=&\sum_{i=1}^N\sum_{t=1}^T f_{t,i}(x_{t,i})-\sum_{i=1}^N\sum_{t=1}^T f_{t,i}(x^*)\nonumber\\
    =&\underbrace{\sum_{i=1}^N\sum_{t=1}^T f_{t,i}(x_{t,i})-\sum_{i=1}^N\sum_{t=1}^T f_{t,i}(\Bar{x}_t)}_{R_1}+\underbrace{\sum_{i=1}^N\sum_{t=1}^T f_{t,i}(\Bar{x}_t)-\sum_{i=1}^N\sum_{t=1}^T f_{t,i}(x^*)}_{R_2},
\end{align}
where $x^*=\arg\min F(x)$. It can be seen from \eqref{regretdecomp} that the regret can be decomposed into  two terms: 1) $R_1$, the regret  resulted by the difference between local model $x_{t,i}$ and the global average $\Bar{x}_t$, and 2) $R_2$, the regret
 accumulated over the iteration of the global average $\Bar{x}_t$.

\subsection{Analysis of $R_1$}

To analyze $R_1$, we first have the following lemma to characterize the relationship between the regret and the consensus gap between model parameters.

\lemmarestate*

\begin{proof}
For any $x$, we can have
\begin{align}\label{consensusregret}
    \frac{1}{N}\sum_{i=1}^N f_{t,i}(x)\leq& \frac{1}{N}\sum_{i=1}^N \left\{f_{t,i}(x_{t,i})+\langle \nabla f_{t,i}(x_{t,i}), x-x_{t,i}\rangle+\frac{L}{2}\|x-x_{t,i}\|^2\right\}\nonumber\\
    =&\frac{1}{N}\sum_{i=1}^N\left\{f_{t,i}(x_{t,i})+\langle \nabla f_{t,i}(x_{t,i}), \Bar{x}_t-x_{t,i}\rangle+\langle \nabla f_{t,i}(x_{t,i}), x-\Bar{x}_t\rangle+\frac{L}{2}\|x-x_{t,i}\|^2\right\}\nonumber\\
    =& \frac{1}{N}\sum_{i=1}^N \left\{f_{t,i}(x_{t,i})+\langle \nabla f_{t,i}(x_{t,i}), \Bar{x}_t-x_{t,i}\rangle\right\}+\langle g_t, x-\Bar{x}_t\rangle+\frac{L}{2N}\sum_{i=1}^N \|x-x_{t,i}\|^2\nonumber\\
    =& \frac{1}{N}\sum_{i=1}^N \left\{f_{t,i}(x_{t,i})+\langle \nabla f_{t,i}(x_{t,i}), \Bar{x}_t-x_{t,i}\rangle\right\}+\langle g_t, x-\Bar{x}_t\rangle+\frac{L}{2N}\sum_{i=1}^N \|x-\Bar{x}_t+\Bar{x}_t-x_{t,i}\|^2\nonumber\\
    \leq& \frac{1}{N}\sum_{i=1}^N \left\{f_{t,i}(x_{t,i})+\langle \nabla f_{t,i}(x_{t,i}), \Bar{x}_t-x_{t,i}\rangle\right\}+\langle g_t, x-\Bar{x}_t\rangle+ L\|x-\Bar{x}_t\|^2+\frac{L}{N}\|x_t-\bone \Bar{x}_t\|^2.
\end{align}

Moreover, based on the convexity and smoothness of $f_{t,i}$, it can be shown that
\begin{align*}
    f_{t,i}(\Bar{x}_t)\geq& f_{t,i}(x_{t,i})+\langle \nabla f_{t,i}(x_{t,i}),\Bar{x}_t-x_{t,i}\rangle+\frac{1}{2L}\|\nabla f_{t,i}(\Bar{x}_t)-\nabla f_{t,i}(x_{t,i})\|^2\\
    \geq&f_{t,i}(x_{t,i})+\langle \nabla f_{t,i}(x_{t,i}),\Bar{x}_t-x_{t,i}\rangle.
\end{align*}
Continuing with \eqref{consensusregret} and taking expectation at both sides, it follows that
\begin{align*}
    \mathbb{E}\left[\frac{1}{N}\sum_{i=1}^N f_{t,i}(x)\right]\leq& \mathbb{E}\left[\frac{1}{N}\sum_{i=1}^N f_{t,i}(\Bar{x}_t)+ \langle g_t, x-\Bar{x}_t\rangle+ L\|x-\Bar{x}_t\|^2+\frac{L}{N}\|x_t-\bone \Bar{x}_t\|^2\right],
\end{align*}
such that
\begin{align}\label{expectation}
    \mathbb{E}[F(x)]\leq \mathbb{E}[F(\Bar{x}_t)]+\mathbb{E}[\langle g_t, x-\Bar{x}_t\rangle]+L\mathbb{E}[\|x-\Bar{x}_t\|^2]+\frac{L}{N}\mathbb{E}[\|x_t-\bone \Bar{x}_t\|^2].
\end{align}

Since \eqref{expectation} holds for any $x$, we can have
\begin{align*}
    \mathbb{E}[F(x_{t,i})]\leq \mathbb{E}[F(\Bar{x}_t)]+\mathbb{E}[\langle g_t, x_{t,i}-\Bar{x}_t\rangle]+L\mathbb{E}[\|x_{t,i}-\Bar{x}_t\|^2]+\frac{L}{N}\mathbb{E}[\|x_t-\bone \Bar{x}_t\|^2],
\end{align*}
which indicates that
\begin{align*}
    \frac{1}{N}\sum_{i=1}^N \mathbb{E}[F(x_{t,i})]\leq&
    \frac{1}{N}\sum_{i=1}^N \left\{\mathbb{E}[F(\Bar{x}_t)]+\mathbb{E}[\langle g_t, x_{t,i}-\Bar{x}_t\rangle]+L\mathbb{E}[\|x_{t,i}-\Bar{x}_t\|^2]+\frac{L}{N}\mathbb{E}[\|x_t-\bone \Bar{x}_t\|^2]\right\}\\
    =& \frac{1}{N}\sum_{i=1}^N \mathbb{E}[F(\Bar{x}_t)]+\frac{1}{N}\sum_{i=1}^N\mathbb{E}[\langle g_t, x_{t,i}-\Bar{x}_t\rangle]+\frac{L}{N}\sum_{i=1}^N  \mathbb{E}[\|x_{t,i}-\Bar{x}_t\|^2]+\frac{L}{N}\mathbb{E}[\|x_t-\bone \Bar{x}_t\|^2]\\
    =& \frac{1}{N}\sum_{i=1}^N \mathbb{E}[F(\Bar{x}_t)]+\mathbb{E}\left[\langle g_t, \frac{1}{N}\sum_{i=1}^N(x_{t,i}-\Bar{x}_t)\rangle\right]+\mathbb{E}\left[\frac{L}{N}\sum_{i=1}^N \|x_{t,i}-\Bar{x}_t\|^2\right]+\frac{L}{N}\mathbb{E}[\|x_t-\bone \Bar{x}_t\|^2]\\
    =& \frac{1}{N}\sum_{i=1}^N \mathbb{E}[F(\Bar{x}_t)]+\frac{2L}{N}\mathbb{E}[\|x_t-\bone \Bar{x}_t\|^2].
\end{align*}

Therefore, $R_1$ can be bounded above as follows:
\begin{align*}
    \mathbb{E}\left[\sum_{i=1}^N\sum_{t=1}^T f_{t,i}(x_{t,i})-\sum_{i=1}^N\sum_{t=1}^T f_{t,i}(\Bar{x}_t)\right]=&\sum_{i=1}^N\sum_{t=1}^T \mathbb{E}[F(x_{t,i})]- \sum_{i=1}^N\sum_{t=1}^T \mathbb{E}[F(\Bar{x}_t)]\\
    \leq & 2L\sum_{t=1}^T \mathbb{E}[\|x_t-\bone \Bar{x}_t\|^2],
\end{align*}
thereby completing the proof of Lemma \ref{lemma3}.
\end{proof}

Based on Lemma \ref{lemma3}, to analyze $R_1$, it suffices to analyze the consensus error  $\mathbb{E}[\|x_t-\bone \Bar{x}_t\|^2]$. To this end,
it can be first seen that the average of $s_{t,i}$, i.e., $\Bar{s}_t$, is equal to the global stochastic gradient average $g_t=\frac{1}{N}\sum_{i=1}^N f_{t,i}(x_{t,i})$ from Lemma \ref{lemma1}. Since $s_{t,i}$ is designed to estimate the global gradient average $g_t$, it is necessary to quantify the estimation gap $\|s_t-\bone g_t\|$. Through careful manipulations, we have the following result regarding the consensus error.

\begin{lemma}\label{lemma4}
For any $\beta>0$, we have the following result:
\begin{align*}
    \begin{bmatrix}
    \mathbb{E}[\|x_{t+1}-\bone \Bar{x}_{t+1}\|^2]\\
    \mathbb{E}[\|s_{t+1}-\bone g_{t+1}\|^2]
    \end{bmatrix}
    \leq& \begin{bmatrix}
    \frac{1+\rho^2}{2} & \eta^2\frac{1+\rho^2}{1-\rho^2} \\
    (2+\frac{1}{\beta})L^2\|W-I\|^2 & \rho^2+2\beta\rho^2+2L\rho\eta+2\eta^2L^2
    \end{bmatrix}
    \cdot \begin{bmatrix}
    \mathbb{E}[\|x_{t}-\bone \Bar{x}_{t}\|^2]\\
    \mathbb{E}[\|s_{t}-\bone g_{t}\|^2]
    \end{bmatrix}\nonumber\\
    &+\begin{bmatrix}
    0\\
    (2+\frac{1}{\beta})N\eta^2L^2\mathbb{E}[\|g_t\|^2]+2\sigma^2+2\eta LN\sigma^2+2N\sigma^2
    \end{bmatrix}.
\end{align*}
\end{lemma}

\begin{proof}
Based on the update rule, we can have
\begin{align}\label{boundx}
    \|x_{t+1}-\bone\Bar{x}_{t+1}\|^2\leq&\|Wx_t-\eta s_t-\bone(\Bar{x}_t-\eta g_t)\|^2\nonumber\\
    =& \|Wx_t-\bone\Bar{x}_t\|^2-2\eta\langle Wx_t-\bone\Bar{x}_t, s_t-\bone g_t\rangle+\eta^2\|s_t-\bone g_t\|^2\nonumber\\
    \leq& \rho^2\|x_t-\bone\Bar{x}_t\|^2+\eta\left[\frac{1-\rho^2}{2\eta\rho^2}\|Wx_t-\bone\Bar{x}_t\|^2+\frac{2\eta\rho^2}{1-\rho^2}\|s_t-\bone g_t\|^2\right]+\eta^2\|s_t-\bone g_t\|^2\nonumber\\
    \leq& \rho^2\|x_t-\bone\Bar{x}_t\|^2+\frac{(1-\rho^2)\rho^2}{2\rho^2}\|x_t-\bone\Bar{x}_t\|^2+\frac{2\eta^2\rho^2}{1-\rho^2}\|s_t-\bone g_t\|^2+\eta^2\|s_t-\bone g_t\|^2\nonumber\\
    \leq& \frac{1+\rho^2}{2}\|x_t-\bone\Bar{x}_t\|^2+\eta^2\frac{1+\rho^2}{1-\rho^2}\|s_t-\bone g_t\|^2.
\end{align}

Besides, for the global gradient estimation gap, it follows that
\begin{align}\label{bounds}
    \|s_{t+1}-\bone g_{t+1}\|^2=& \|Ws_t+\nabla_{t+1}-\nabla_t-\bone g_{t+1}\|^2\nonumber\\
    =&\|Ws_t-\bone g_t+\nabla_{t+1}-\nabla_t+\bone g_t-\bone g_{t+1}\|^2\nonumber\\
    =&\|Ws_t-\bone g_t\|^2+2\langle Ws_t-\bone g_t, \nabla_{t+1}-\nabla_t+\bone g_t-\bone g_{t+1}\rangle+\|\nabla_{t+1}-\nabla_t+\bone g_t-\bone g_{t+1}\|^2\nonumber\\
    =& \|Ws_t-\bone g_t\|^2+\|\nabla_{t+1}-\nabla_t+\bone g_t-\bone g_{t+1}\|^2+2\langle Ws_t-\bone g_t, \nabla_{t+1}-\nabla_t\rangle\nonumber\\ &+2\langle Ws_t-\bone g_t, \bone g_t-\bone g_{t+1}\rangle\nonumber\\
    \leq& \rho^2 \|s_t-\bone g_t\|^2+\|\nabla_{t+1}-\nabla_t\|^2+2\langle Ws_t-\bone g_t, \nabla_{t+1}-\nabla_t\rangle
\end{align}
where the last inequality holds because
\begin{align*}
    &\|\nabla_{t+1}-\nabla_t+\bone g_{t}-\bone g_{t+1}\|^2\\
    =&\|\nabla_{t+1}-\nabla_t\|^2+n\|g_{t+1}-g_{t}\|^2-2\langle \nabla_{t+1}-\nabla_t, \bone g_{t+1}-\bone g_t\rangle\\
    \leq & \|\nabla_{t+1}-\nabla_t\|^2-n\|g_{t+1}-g_{t}\|^2\\
    \leq & \|\nabla_{t+1}-\nabla_t\|^2,
\end{align*}
and
\begin{align*}
    \langle Ws_t-\bone g_t, \bone g_t-\bone g_{t+1}\rangle=&\sum_{i=1}^N\left\langle \sum_{j=1}^N w_{ij}s_{t,j}-g_t, g_t-g_{t+1}\right\rangle\\
    =&\left\langle\sum_{i=1}^N\sum_{j=1}^N w_{ij}s_{t,j}-Ng_t, g_t-g_{t+1}\right\rangle\\
    =& \left\langle\sum_{i=1}^N\sum_{j=1}^N w_{ij}s_{t,j}-\sum_{i=1}^N s_{t,i}, g_t-g_{t+1}\right\rangle\\
    =&0.
\end{align*}

We next bound the three terms in \eqref{bounds} separately. 
To bound $\|\nabla_{t+1}-\nabla_t\|^2$, we consider the conditional expectation  $\mathbb{E}[\|\nabla_{t+1}-\nabla_t\|^2|\mathcal{F}_t]$. It is clear that
\begin{align*}
    \mathbb{E}[\|\nabla_{t+1}-\nabla_t\|^2|\mathcal{F}_t]
    =&\mathbb{E}[\|G_{t+1}-G_t+\nabla_{t+1}-G_{t+1}+G_t-\nabla_t\|^2|\mathcal{F}_t]\\
    =&\mathbb{E}[\|G_{t+1}-G_t\|^2|\mathcal{F}_t]+2\mathbb{E}[\langle G_{t+1}-G_t, \nabla_{t+1}-G_{t+1}-\nabla_t+G_t \rangle|\mathcal{F}_t]\\
    &+\mathbb{E}[\|\nabla_{t+1}-G_{t+1}-\nabla_t+G_t\|^2|\mathcal{F}_t]\\
    \leq&  \mathbb{E}[\|G_{t+1}-G_t\|^2|\mathcal{F}_t]+2\mathbb{E}[\langle G_{t+1}, G_t-\nabla_t\rangle|\mathcal{F}_t]+2N\sigma^2.
\end{align*}
Since
\begin{align*}
    \|G_{t+1}-G_t\|^2\leq& L^2\|x_{t+1}-x_t\|^2\nonumber\\
    \leq&L^2\|Wx_t-\eta s_t-x_t\|^2\nonumber\\
    =&L^2(\|(W-I)(x_t-\bone\Bar{x}_t)\|^2-2\langle(W-I)(x_t-\bone\Bar{x}_t), \eta s_t\rangle+\eta^2\|s_t\|^2)\nonumber\\
    \leq& 2\|W-I\|^2L^2\|x_t-\bone\Bar{x}_t\|^2+2\eta^2 L^2\|s_t-\bone g_t+\bone g_t\|^2\nonumber\\
    \leq& 2\|W-I\|^2L^2\|x_t-\bone\Bar{x}_t\|^2+2\eta^2L^2\|s_t-\bone g_t\|^2+2N\eta^2L^2\|g_t\|^2,
\end{align*}
we use Lemma \ref{lemma_22} to conclude that
\begin{align}\label{bounddelta}
    &\mathbb{E}[\|\nabla_{t+1}-\nabla_t\|^2|\mathcal{F}_t]\nonumber\\
    \leq& 2\|W-I\|^2L^2\mathbb{E}[\|x_t-\bone\Bar{x}_t\|^2|\mathcal{F}_t]+2\eta^2L^2\mathbb{E}[\|s_t-\bone g_t\|^2|\mathcal{F}_t]+2N\eta^2L^2\mathbb{E}[\|g_t\|^2|\mathcal{F}_t]+2\eta LN\sigma^2+2N\sigma^2.
\end{align}

To bound $\langle Ws_t-\bone g_t, \nabla_{t+1}-\nabla_t\rangle$, we consider the conditional expectation $\mathbb{E}[\langle Ws_t-\bone g_t, \nabla_{t+1}-\nabla_t\rangle|\mathcal{F}_t]$ given $\mathcal{F}_t$, such that
\begin{align}\label{expect}
    &\mathbb{E}[\langle Ws_t-\bone g_t, \nabla_{t+1}-\nabla_t\rangle|\mathcal{F}_t]\nonumber\\
    =&\mathbb{E}[\mathbb{E}[\langle Ws_t-\bone g_t, \nabla_{t+1}-\nabla_t\rangle|\mathcal{F}_{t+1}]|\mathcal{F}_t]\nonumber\\
    =&\mathbb{E}[\langle Ws_t-\bone g_t, G_{t+1}-\nabla_t\rangle|\mathcal{F}_t]\nonumber\\
    =&\mathbb{E}[\langle Ws_t-\bone g_t, G_{t+1}-G_t\rangle|\mathcal{F}_t]+\mathbb{E}[\langle Ws_t-\bone g_t, G_{t}-\nabla_t\rangle|\mathcal{F}_t].
\end{align}

For the term $\mathbb{E}[\langle Ws_t-\bone g_t, G_{t+1}-G_t\rangle|\mathcal{F}_t]$, we can obtain
\begin{align}\label{boundinner}
    &\mathbb{E}[\langle Ws_t-\bone g_t, G_{t+1}-G_t\rangle|\mathcal{F}_t]\nonumber\\
    \leq& L\rho\mathbb{E}[\|s_t-\bone g_t\|\|x_{t+1}-x_t\||\mathcal{F}_t]\nonumber\\
    \leq& L\rho\mathbb{E}[\|s_t-\bone g_t\|\|Wx_t-\eta s_t-x_t\||\mathcal{F}_t]\nonumber\\
    =& L\rho\mathbb{E}[\|s_t-\bone g_t\|\|(W-I)(x_t-\bone\Bar{x}_t)-\eta(s_t-\bone g_t+\bone g_t)\||\mathcal{F}_t]\nonumber\\
    \leq& L\rho\mathbb{E}[\|s_t-\bone g_t\|(\|W-I\|\|x_t-\bone\Bar{x}_t\|+\eta\|s_t-\bone g_t\|+\eta\|\bone g_t\|)|\mathcal{F}_t]\nonumber\\
    =& L\rho\mathbb{E}[\|W-I\|\|s_t-\bone g_t\|\|x_t-\bone\Bar{x}_t\|+\eta\|s_t-\bone g_t\|^2+\eta\sqrt{N}\|s_t-\bone g_t\|\|g_t\||\mathcal{F}_t].
\end{align}

For the term $\mathbb{E}[\langle Ws_t-\bone g_t, G_{t}-\nabla_t\rangle|\mathcal{F}_t]$, it follows from Lemma \ref{lemma2} that
\begin{align}\label{boundpu}
    \mathbb{E}[\langle Ws_t-\bone g_t, G_{t}-\nabla_t\rangle|\mathcal{F}_t]\leq \sigma^2.
\end{align}

Therefore, by substituting \eqref{boundinner} and \eqref{boundpu} into \eqref{expect}, we have
\begin{align}\label{boundinnerdelta}
    &\mathbb{E}[\langle Ws_t-\bone g_t, \nabla_{t+1}-\nabla_t\rangle|\mathcal{F}_t]\nonumber\\
    \leq& L\rho\mathbb{E}[\|W-I\|\|s_t-\bone g_t\|\|x_t-\bone\Bar{x}_t\||\mathcal{F}_t]+L\rho\eta\mathbb{E}[\|s_t-\bone g_t\|^2|\mathcal{F}_t]+L\rho\eta\sqrt{N}\mathbb{E}[\|s_t-\bone g_t\|\|g_t\||\mathcal{F}_t]+\sigma^2.
\end{align}

Combining \eqref{boundinnerdelta} with \eqref{bounds} and \eqref{bounddelta}, we conclude that  for any $\beta>0$, 
\begin{align*}
    &\mathbb{E}[\|s_{t+1}-\bone g_{t+1}\|^2|\mathcal{F}_t]\\
    \leq& \rho^2\mathbb{E}[\|s_t-\bone g_t\|^2|\mathcal{F}_t]+\mathbb{E}[2\|W-I\|^2L^2\|x_t-\bone\Bar{x}_t\|^2+2\eta^2L^2\|s_t-\bone g_t\|^2+2N\eta^2L^2\|g_t\|^2|\mathcal{F}_t]\\
    &+2L\rho\mathbb{E}[\|W-I\|\|s_t-\bone g_t\|\|x_t-\bone\Bar{x}_t\||\mathcal{F}_t]+2L\rho\eta\mathbb{E}[\|s_t-\bone g_t\|^2|\mathcal{F}_t]\\
    &+2L\rho\eta\sqrt{N}\mathbb{E}[\|s_t-\bone g_t\|\|g_t\||\mathcal{F}_t]+2\sigma^2+2\eta LN\sigma^2+2N\sigma^2\\
    \leq& (\rho^2+\beta\rho^2+2L\rho\eta+2\eta^2L^2)\mathbb{E}[\|s_t-\bone g_t\|^2|\mathcal{F}_t]+(2+\frac{1}{\beta})L^2\|W-I\|^2\mathbb{E}[\|x_t-\bone\Bar{x}_t\|^2|\mathcal{F}_t]\\
    &+2N\eta^2L^2\mathbb{E}[\|g_t\|^2|\mathcal{F}_t]+2L\rho\eta\sqrt{N}\mathbb{E}[\|s_t-\bone g_t\|\|g_t\||\mathcal{F}_t]+2\sigma^2+2\eta LN\sigma^2+2N\sigma^2\\
    \leq& (\rho^2+2\beta\rho^2+2L\rho\eta+2\eta^2L^2)\mathbb{E}[\|s_t-\bone g_t\|^2|\mathcal{F}_t]+(2+\frac{1}{\beta})L^2\|W-I\|^2\mathbb{E}[\|x_t-\bone\Bar{x}_t\|^2|\mathcal{F}_t]\\
    &+(2+\frac{1}{\beta})N\eta^2L^2\mathbb{E}[\|g_t\|^2|\mathcal{F}_t]+2\sigma^2+2\eta LN\sigma^2+2N\sigma^2
\end{align*}


Taking expectation and combining with \eqref{boundx}, we can obtain the following result:
\begin{align}\label{system}
    \begin{bmatrix}
    \mathbb{E}[\|x_{t+1}-\bone \Bar{x}_{t+1}\|^2]\\
    \mathbb{E}[\|s_{t+1}-\bone g_{t+1}\|^2]
    \end{bmatrix}
    \leq& \begin{bmatrix}
    \frac{1+\rho^2}{2} & \eta^2\frac{1+\rho^2}{1-\rho^2} \\
    (2+\frac{1}{\beta})L^2\|W-I\|^2 & \rho^2+2\beta\rho^2+2L\rho\eta+2\eta^2L^2
    \end{bmatrix}
    \cdot \begin{bmatrix}
    \mathbb{E}[\|x_{t}-\bone \Bar{x}_{t}\|^2]\\
    \mathbb{E}[\|s_{t}-\bone g_{t}\|^2]
    \end{bmatrix}\nonumber\\
    &+\begin{bmatrix}
    0\\
    (2+\frac{1}{\beta})N\eta^2L^2\mathbb{E}[\|g_t\|^2]+2\sigma^2+2\eta LN\sigma^2+2N\sigma^2
    \end{bmatrix}\nonumber\\
    \triangleq& A_t\begin{bmatrix}
    \mathbb{E}[\|x_{t}-\bone \Bar{x}_{t}\|^2]\\
    \mathbb{E}[\|s_{t}-\bone g_{t}\|^2]
    \end{bmatrix}+B_t
\end{align}
where 
\begin{align*}
    A_t=\begin{bmatrix}
    \frac{1+\rho^2}{2} & \eta^2\frac{1+\rho^2}{1-\rho^2} \\
    (2+\frac{1}{\beta})L^2\|W-I\|^2 & \rho^2+2\beta\rho^2+2L\rho\eta+2\eta^2L^2
    \end{bmatrix}
\end{align*}
 and 
 \begin{align*}
     B_t=\begin{bmatrix}
    0\\
    (2+\frac{1}{\beta})N\eta^2L^2\mathbb{E}[\|g_t\|^2]+2\sigma^2+2\eta LN\sigma^2+2N\sigma^2
    \end{bmatrix},
 \end{align*}
completing the proof of Lemma \ref{lemma4}.
    
\end{proof}
    
Telescoping \eqref{system}, we have for $t\geq 2$
\begin{align}
    \begin{bmatrix}
    \mathbb{E}[\|x_{t}-\bone \Bar{x}_{t}\|^2]\\
    \mathbb{E}[\|s_{t}-\bone g_{t}\|^2]
    \end{bmatrix}
    \leq \left(\prod_{l=1}^{t-1} A_l\right)\begin{bmatrix}
    \mathbb{E}[\|x_{1}-\bone \Bar{x}_{1}\|^2]\\
    \mathbb{E}[\|s_{1}-\bone g_{1}\|^2]
    \end{bmatrix}
    +\sum_{l=1}^{t-1} \left(B_l\prod_{j=l+1}^{t-1}A_j\right)
\end{align}
where $\mathbb{E}[\|x_t-\bone\Bar{x}_t\|^2]$ and $\mathbb{E}[\|s_t-\bone g_t\|^2]$ would converge to a neighborhood of $0$, if the spectral radius of $A_t$, i.e., $\rho_s(A_t)$, is smaller than $1$ for any $t\geq 1$. The next lemma characterizes the conditions for $\rho_s(A_t)<1$.

\begin{lemma}\label{lemma5}
When $\eta\leq \frac{(1-\rho^2)^{1.5}}{32L\sqrt{1+\rho^2}}$ and $\beta=\frac{1-\rho^2}{4\rho^2}-\frac{L\eta}{\rho}-\frac{L^2\eta^2}{\rho^2}$, we can have
\begin{align*}
    \rho_s(A_t)\leq \frac{3+\rho^2}{4}<1.
\end{align*}
\end{lemma}

\begin{proof}
To ensure $\rho_s(A_t)<1$, the eigenvalues $\lambda$ of $A_t$, i.e., the solutions of $\det(A_t-\lambda I)=0$, must be smaller than $1$. By computing $\det(A_t-\lambda I)$, we first need the following holds for some $\beta>0$:
\begin{align*}
    \frac{1+\rho^2}{2}=\rho^2+2\beta\rho^2+2L\rho\eta+2\eta^2L^2<1.
\end{align*}

Clearly, when $\eta\leq\frac{(1-\rho^2)^{1.5}}{32L\sqrt{1+\rho^2}}\leq \frac{1-\rho^2}{8L}$, it can be seen that the selection of $\beta$ satisfying the above equality is positive:
\begin{align*}
    \beta=&\frac{1-\rho^2}{4\rho^2}-\frac{L\eta}{\rho}-\frac{L^2\eta^2}{\rho^2}\\
    \geq& \frac{1-\rho^2}{4\rho^2}-\frac{1-\rho^2}{8\rho}-\frac{(1-\rho^2)^2}{64\rho^2}\\
    \geq & \frac{3}{32\rho^2}>0.
\end{align*}

Therefore, for any $\lambda$ satisfying $\det(A_t-\lambda I)=0$, the following is true:
\begin{align*}
    \lambda\leq&\frac{1}{2}\left(\frac{1+\rho^2}{2}+\rho^2+2\beta\rho^2+2L\rho\eta+2L^2\eta^2+2L\eta\sqrt{\frac{1+\rho^2}{1-\rho^2}(8+\frac{4}{\beta})}\right)\\
    =&\frac{1+\rho^2}{2}+L\eta\sqrt{\frac{1+\rho^2}{1-\rho^2}(8+\frac{4}{\beta})}\\
    \leq& \frac{1+\rho^2}{2}+L\eta\sqrt{\frac{1+\rho^2}{1-\rho^2}(8+44\rho^2)}\\
    \leq& \frac{1+\rho^2}{2}+L\frac{(1-\rho^2)^{1.5}}{32L\sqrt{1+\rho^2}}\sqrt{\frac{1+\rho^2}{1-\rho^2}(8+44\rho^2)}\\
    \leq& \frac{3+\rho^2}{4}<1.
\end{align*}
\end{proof}

Next, we are ready to prove Lemma \ref{lemma_3}.
\lemmarestatea*
\begin{proof}
Let $\alpha=\frac{3+\rho^2}{4}$. With the same spirit in \cite{qu2017harnessing}, we can show that
\begin{align}\label{finalx}
    \mathbb{E}[\|x_t-\bone\Bar{x}_t\|^2]\leq A_1\alpha^{t-1}+A_2\eta^2\frac{1+\rho^2}{1-\rho^2}(6N\eta^2L^2\mathbb{E}[\|g_t\|^2]+\sigma^2+\eta LN\sigma^2+N\sigma^2),
\end{align}
for some $A_1$ and $A_2$. 
This can be achieved through diagonalization of $A$, which needs tedious calculations. To start with, for the selection of $\beta$, we first have
\begin{align}\label{reshape}
    &\mathbb{E}[\|s_{t+1}-\bone g_{t+1}\|^2|\mathcal{F}_t]\nonumber\\
    \leq& \frac{1+\rho^2}{2}\mathbb{E}[\|s_t-\bone g_t\|^2|\mathcal{F}_t]+(2+\frac{1}{\beta})L^2\|W-I\|^2\mathbb{E}[\|x_t-\bone\Bar{x}_t\|^2|\mathcal{F}_t]
    +(2+\frac{1}{\beta})N\eta^2L^2\mathbb{E}[\|g_t\|^2|\mathcal{F}_t]+2\sigma^2\nonumber\\
    \leq& \left[\frac{1+\rho^2}{2}+2L^2\eta^2\frac{1+\rho^2}{1-\rho^2}(2+\frac{1}{\beta})\right]\mathbb{E}[\|s_t-\bone g_t\|^2|\mathcal{F}_t]+(2+\frac{1}{\beta})L^2\|W-I\|^2\mathbb{E}[\|x_t-\bone\Bar{x}_t\|^2|\mathcal{F}_t]\nonumber\\
    &+(2+\frac{1}{\beta})N\eta^2L^2\mathbb{E}[\|g_t\|^2|\mathcal{F}_t]+2\sigma^2+2\eta LN\sigma^2+2N\sigma^2.
\end{align}
Combining \eqref{reshape} and \eqref{boundx}, we can obtain
\begin{align*}
    \begin{bmatrix}
    \mathbb{E}[\|s_{t+1}-\bone g_{t+1}\|^2]\\
    \mathbb{E}[\|x_{t+1}-\bone \Bar{x}_{t+1}\|^2]
    \end{bmatrix}
    \leq& \begin{bmatrix}
    \frac{1+\rho^2}{2}+2L^2\eta^2\frac{1+\rho^2}{1-\rho^2}(2+\frac{1}{\beta}) & 4L^2(2+\frac{1}{\beta}) \\
    \eta^2\frac{1+\rho^2}{1-\rho^2} & \frac{1+\rho^2}{2}
    \end{bmatrix}
    \cdot \begin{bmatrix}
        \mathbb{E}[\|s_{t}-\bone g_{t}\|^2]\\
    \mathbb{E}[\|x_{t}-\bone \Bar{x}_{t}\|^2]
    \end{bmatrix}\nonumber\\
    &+\begin{bmatrix}
    (2+\frac{1}{\beta})N\eta^2L^2\mathbb{E}[\|g_t\|^2]+2\sigma^2+2\eta LN\sigma^2+2N\sigma^2\\
    0
    \end{bmatrix}.
\end{align*}
For the matrix $\Tilde{A}=\begin{bmatrix}
    \frac{1+\rho^2}{2}+2L^2\eta^2\frac{1+\rho^2}{1-\rho^2}(2+\frac{1}{\beta}) & 4L^2(2+\frac{1}{\beta}) \\
    \eta^2\frac{1+\rho^2}{1-\rho^2} & \frac{1+\rho^2}{2}
    \end{bmatrix}$, we can diagonalize it as $\Tilde{A}=V\Lambda V^{-1}$, where $\Sigma=\begin{bmatrix}
        \lambda_1 & 0\\
        0 & \lambda_2
        \end{bmatrix}$ with
        \begin{align*}
            \lambda_1&=\frac{1+\rho^2+2L^2\eta^2\frac{1+\rho^2}{1-\rho^2}(2+\frac{1}{\beta})-\sqrt{4L^4\eta^4\left(\frac{1+\rho^2}{1-\rho^2}\right)^2(2+\frac{1}{\beta})^2+16L^2\eta^2\frac{1+\rho^2}{1-\rho^2}(2+\frac{1}{\beta})}}{2},\\
            \lambda_2&=\frac{1+\rho^2+2L^2\eta^2\frac{1+\rho^2}{1-\rho^2}(2+\frac{1}{\beta})+\sqrt{4L^4\eta^4\left(\frac{1+\rho^2}{1-\rho^2}\right)^2(2+\frac{1}{\beta})^2+16L^2\eta^2\frac{1+\rho^2}{1-\rho^2}(2+\frac{1}{\beta})}}{2}.
        \end{align*}
And matrix $V$ and $V^{-1}$ are 
\begin{align*}
    V=\begin{bmatrix}
    \frac{2L^2\eta\sqrt{\frac{1+\rho^2}{1-\rho^2}}(2+\frac{1}{\beta})-L\sqrt{2(2+\frac{1}{\beta})[8+2L^2\eta^2\frac{1+\rho^2}{1-\rho^2}(2+\frac{1}{\beta})]}}{2\eta\sqrt{\frac{1+\rho^2}{1-\rho^2}}}& \frac{2L^2\eta\sqrt{\frac{1+\rho^2}{1-\rho^2}}(2+\frac{1}{\beta})+L\sqrt{2(2+\frac{1}{\beta})[8+2L^2\eta^2\frac{1+\rho^2}{1-\rho^2}(2+\frac{1}{\beta})]}}{2\eta\sqrt{\frac{1+\rho^2}{1-\rho^2}}}\\
    1&1
    \end{bmatrix}
\end{align*}
and
\begin{align*}
    V^{-1}=\begin{bmatrix}
    -\frac{\eta\sqrt{\frac{1+\rho^2}{1-\rho^2}}}{L\sqrt{2(2+\frac{1}{\beta})[8+2L^2\eta^2\frac{1+\rho^2}{1-\rho^2}(2+\frac{1}{\beta})]}} &
    \frac{1}{2}+\frac{1}{2}\sqrt{\frac{2L^2\eta^2\frac{1+\rho^2}{1-\rho^2}(2+\frac{1}{\beta})}{8+2L^2\eta^2\frac{1+\rho^2}{1-\rho^2}(2+\frac{1}{\beta})}}\\
    \frac{\eta\sqrt{\frac{1+\rho^2}{1-\rho^2}}}{L\sqrt{2(2+\frac{1}{\beta})[8+2L^2\eta^2\frac{1+\rho^2}{1-\rho^2}(2+\frac{1}{\beta})]}} &
    \frac{1}{2}-\frac{1}{2}\sqrt{\frac{2L^2\eta^2\frac{1+\rho^2}{1-\rho^2}(2+\frac{1}{\beta})}{8+2L^2\eta^2\frac{1+\rho^2}{1-\rho^2}(2+\frac{1}{\beta})}}
    \end{bmatrix}.
\end{align*}
Let $\Tilde{B}_1=\begin{bmatrix}
        \mathbb{E}[\|s_{t}-\bone g_{t}\|^2]\\
    \mathbb{E}[\|x_{t}-\bone \Bar{x}_{t}\|^2]
    \end{bmatrix}$.
We can show that the second row of  $\Tilde{A}^p\Tilde{B}_1$ is smaller than $C\alpha^p$ for some constant $C$.

Moreover, it can be shown that 
\begin{align*}
    \sum_{p=0}^{t-1}\Tilde{A}^p=(I-\Tilde{A})^{-1}(I-\Tilde{A}^{t-1})=\frac{1}{\det(I-\Tilde{A})}\adj(I-\Tilde{A})(I-\Tilde{A}^{t-1})
\end{align*}
where $\det(I-\Tilde{A})>0$ since $\rho(\Tilde{A})<1$, and $\adj(I-\Tilde{A})$ is the adjugate matrix of $I-\Tilde{A}$. 

Let $\Tilde{B}_t=\begin{bmatrix}
    (2+\frac{1}{\beta})N\eta^2L^2\mathbb{E}[\|g_t\|^2]+2\sigma^2+2\eta LN\sigma^2+2N\sigma^2\\
    0
    \end{bmatrix}$. 
The following holds for
the second row of $\sum_{p=0}^{t-1}\Tilde{A}^p\Tilde{B}_l$:
\begin{align*}
    \sum_{p=0}^{t-1}\Tilde{A}^p\Tilde{B}_l\leq \frac{1}{\det(I-\Tilde{A})}\eta^2\frac{1+\rho^2}{1-\rho^2}[(2+\frac{1}{\beta})N\eta^2L^2\mathbb{E}[\|g_t\|^2]+2\sigma^2+2\eta LN\sigma^2+2N\sigma^2].
\end{align*}
 Let $A_1=C$ and $A_2=\frac{1}{\det(I-\Tilde{A})}$, we can obtain \eqref{finalx}.

Next, we need to bound the term $\mathbb{E}[\|g_t\|^2]$. It follows that 
\begin{align*}
    \mathbb{E}[\|g_t\|^2]=&\mathbb{E}\left[\left\|\frac{1}{N}\sum_{i=1}^N\nabla f_{t,i}(x_{t,i})\right\|^2\right]\\
    \leq&\frac{1}{N}\sum_{i=1}^N \mathbb{E}[\|\nabla f_{t,i}(x_{t,i})\|^2]\\
    \leq& D.
\end{align*}

For a constant learning rate $\eta$, from \eqref{finalx}, it follows that
\begin{align*}
    \sum_{t=1}^T \mathbb{E}[\|x_t-\bone\Bar{x}_t\|^2]\leq& \sum_{t=2}^T \left[A_1\alpha^{t-1}+ A_2[18\eta^2L^2\sigma^2+18N\eta^2L^2D+(1+\eta LN+N)\sigma^2]\eta^2\frac{1+\rho^2}{1-\rho^2}\right]+\mathbb{E}[\|x_1-\bone\Bar{x}_1\|^2]\\
    =& A_1\frac{\alpha-\alpha^T}{1-\alpha}+A_2(18\eta^2L^2\sigma^2+18N\eta^2L^2D+(1+\eta LN+N)\sigma^2)\eta^2\frac{1+\rho^2}{1-\rho^2}(T-1)+\mathbb{E}[\|x_1-\bone\Bar{x}_1\|^2]\\
    \leq&A_2\eta^2\frac{1+\rho^2}{1-\rho^2}[18\eta^2\sigma^2L^2+18N\eta^2L^2D+(1+\eta LN+N)\sigma^2]T\\
    &+A_1\frac{\alpha-\alpha^T}{1-\alpha}+\mathbb{E}[\|x_1-\bone\Bar{x}_1\|^2].
\end{align*}
\end{proof}
Based on Lemma \ref{lemma3}, we can obtain the upper bound of $R_1$.

\subsection{Analysis of $R_2$}

Next, we analyze $R_2$. First, denote $f_t(\cdot)=\frac{1}{N}\sum_{i=1}^N f_{t,i}(\cdot)$, then we can have
\begin{align*}
    \sum_{i=1}^N\sum_{t=1}^T f_{t,i}(\Bar{x}_t)-\sum_{i=1}^N\sum_{t=1}^T f_{t,i}(x^*)
    =N\sum_{t=1}^T \left[f_t(\Bar{x}_t)-f_t(x^*)\right].
\end{align*}
And the following lemma gives an upper bound on $R_2$.

\lemmarestateb*

\begin{proof}



Following the same line as in \cite{qu2017harnessing}, we can first have 
\begin{align}\label{fhat1}
    f_t(x)\geq \hat{f}_t+\langle g_t, x-\Bar{x}_t\rangle 
\end{align}
and
\begin{align}\label{fhat2}
    f_t(x)\leq \hat{f}_t+\langle g_t, x-\Bar{x}_t\rangle+L\|x-\Bar{x}_t\|^2+\frac{L}{N}\|x_t-\bone\Bar{x}_t\|^2
\end{align}
where
\begin{align*}
    \hat{f}_t=\frac{1}{N}\sum_{i=1}^N[f_{t,i}(x_{t,i})+\langle\nabla f_{t,i}(x_{t,i}), \Bar{x}_t-x_{t,i}\rangle].
\end{align*}

To show this, for \eqref{fhat1}, we have
\begin{align*}
    f_t(x)=&\frac{1}{N}\sum_{i=1}^N f_{t,i}(x)\\
    \geq & \frac{1}{N}\sum_{i=1}^N [f_{t,i}(x_{t,i})+\langle\nabla f_{t,i}(x_{t,i}), x-x_{t,i}\rangle]\\
    =& \frac{1}{N}\sum_{i=1}^N [f_{t,i}(x_{t,i})+\langle\nabla f_{t,i}(x_{t,i}), \Bar{x}_t-x_{t,i}\rangle]+\frac{1}{N}\sum_{i=1}^N \langle \nabla f_{t,i}(x_{t,i}), x-\Bar{x}_t\rangle\\
    =& \hat{f}_t+\langle g_t, x-\Bar{x}_t\rangle,
\end{align*}
and for \eqref{fhat2}, it follows that
\begin{align*}
    f_t(x)=&\frac{1}{N}\sum_{i=1}^N f_{t,i}(x)\\
    \leq& \frac{1}{N}\sum_{i=1}^N [f_{t,i}(x_{t,i})+\langle\nabla f_{t,i}(x_{t,i}), x-x_{t,i}\rangle+\frac{L}{2}\|x-x_{t,i}\|^2]\\
    =& \frac{1}{N}\sum_{i=1}^N[f_{t,i}(x_{t,i})+\langle\nabla f_{t,i}(x_{t,i}), \Bar{x}_t-x_{t,i}\rangle]+\frac{1}{N}\sum_{i=1}^N \langle\nabla f_{t,i}(x_{t,i}),x-\Bar{x}_t\rangle+\frac{L}{2N}\sum_{i=1}^N\|x-x_{t,i}\|^2\\
    &\leq \hat{f}_t+\langle g_t, x-\Bar{x}_t\rangle+L\|x-\Bar{x}_t\|^2+\frac{L}{N}\|x_t-\bone\Bar{x}_t\|^2.
\end{align*}

Next, we can show that
\begin{align*}
    \|\Bar{x}_t-x^*\|^2=&\|\Bar{x}_{t+1}-x^*-\Bar{x}_{t+1}+\Bar{x}_t\|^2\\
    =&\|\Bar{x}_{t+1}-x^*\|^2-2\langle\Bar{x}_{t+1}-\Bar{x}_t,\Bar{x}_{t+1}-x^*\rangle+\|\Bar{x}_{t+1}-\Bar{x}_t\|^2\\
    \overset{(a)}{=}& \|\Bar{x}_{t+1}-x^*\|^2+2\eta\langle g_t, \Bar{x}_{t+1}-x^*\rangle+\eta^2\|g_t\|^2\\
    =& \|\Bar{x}_{t+1}-x^*\|^2+2\eta\langle g_t, \Bar{x}_{t}-x^*\rangle+2\eta\langle g_t, \Bar{x}_{t+1}-\Bar{x}_t\rangle+\eta^2\|g_t\|^2\\
    \overset{(b)}{\geq}& \|\Bar{x}_{t+1}-x^*\|^2+2\eta [\hat{f}_t-f_t(x^*)]+2\eta[\langle g_t, \Bar{x}_{t+1}-\Bar{x}_t\rangle+\frac{\eta}{2}\|g_t\|^2]\\
    \overset{(c)}{\geq}& \|\Bar{x}_{t+1}-x^*\|^2+2\eta [\hat{f}_t-f_t(x^*)]+2\eta [f_t(\Bar{x}_{t+1})-\hat{f}_t+(\frac{\eta}{2}-\eta^2L)\|g_t\|^2-\frac{L}{N}\|x_t-\bone\Bar{x}_t\|^2]\\
    =&\|\Bar{x}_{t+1}-x^*\|^2+2\eta [f_t(\Bar{x}_{t+1})-f_t(x^*)]+2\eta[(\frac{\eta}{2}-\eta^2L)\|g_t\|^2-\frac{L}{N}\|x_t-\bone\Bar{x}_t\|^2]
\end{align*}
Here (a) is based on the update rule, (b) is based on \eqref{fhat1}, and (c) is based on \eqref{fhat2} by setting $x=x^*$. Therefore,
\begin{align}\label{eq:before}
    \sum_{t=1}^{T}[f_t(\Bar{x}_{t+1})-f_t(x^*)]\leq \frac{\|\Bar{x}_1-x^*\|^2}{2\eta}+\sum_{t=1}^{T}[\frac{L}{N}\|x_t-\bone\Bar{x}_t\|^2-(\frac{\eta}{2}-\eta^2L)\|g_t\|^2]
\end{align}
which indicates that
\begin{align}\label{eq:f1}
    \sum_{t=1}^T\mathbb{E}[F(\Bar{x}_{t+1})-F(x^*)]\leq \frac{\|\Bar{x}_1-x^*\|^2}{2\eta}+\frac{L}{N}\sum_{t=1}^T\mathbb{E}[\|x_t-\bone\Bar{x}_t\|^2].
\end{align}

Moreover, it can be seen that
\begin{align}\label{eq:f2}
    \sum_{t=1}^T[F(\Bar{x}_t)-F(\Bar{x}_{t+1})]\leq& \sum_{t=1}^T \langle\nabla F(\Bar{x}_{t+1}),\Bar{x}_t-\Bar{x}_{t+1}\rangle+\frac{L}{2}\sum_{t=1}^T \|\Bar{x}_t-\Bar{x}_{t+1}\|^2\nonumber\\
    =&\eta\sum_{t=1}^T \langle\nabla F(\Bar{x}_{t+1}), g_t\rangle+\frac{\eta^2L}{2}\sum_{t=1}^T\|g_t\|^2\nonumber\\
    \leq& \frac{\eta}{2}\sum_{t=1}^T\|\nabla F(\Bar{x}_{t+1})\|^2+\frac{\eta}{2}\sum_{t=1}^T\|g_t\|^2+\frac{\eta^2L}{2}\sum_{t=1}^T\|g_t\|^2.
\end{align}
Combing \eqref{eq:f1} and \eqref{eq:f2}, we can obtain that
\begin{align*}
    &\sum_{t=1}^T \mathbb{E}[F(\Bar{x}_t)-F(x^*)]\\\leq& \frac{\|\Bar{x}_1-x^*\|^2}{2\eta}+\frac{L}{N}\sum_{t=1}^T\mathbb{E}[\|x_t-\bone\Bar{x}_t\|^2]+\frac{\eta}{2}\sum_{t=1}^T\mathbb{E}[\|\nabla F(\Bar{x}_{t+1})\|^2]+\frac{\eta+\eta^2L}{2}\sum_{t=1}^T\mathbb{E}[\|g_t\|^2]\\
    \leq& \frac{\|\Bar{x}_1-x^*\|^2}{2\eta}+\frac{L}{N}\sum_{t=1}^T\mathbb{E}[\|x_t-\bone\Bar{x}_t\|^2]+\frac{\eta}{2}\mathbb{E}[\|\nabla F(\Bar{x}_{T+1})\|^2]+\frac{\eta}{2}\sum_{t=1}^T\mathbb{E}[\|\nabla F(\Bar{x}_{t})\|^2]+\frac{\eta+\eta^2L}{2}\sum_{t=1}^T\mathbb{E}[\|g_t\|^2]\\
    \leq& \frac{\|\Bar{x}_1-x^*\|^2}{2\eta}+\frac{L}{N}\sum_{t=1}^T\mathbb{E}[\|x_t-\bone\Bar{x}_t\|^2]+\frac{\eta}{2}\mathbb{E}[\|\nabla F(\Bar{x}_{T+1})\|^2]+\eta\sum_{t=1}^T\mathbb{E}[\|\nabla F(\Bar{x}_t)-g_t\|^2]+\frac{3\eta+\eta^2L}{2}\sum_{t=1}^T\mathbb{E}[\|g_t\|^2]\\
    \leq& \frac{\|\Bar{x}_1-x^*\|^2}{2\eta}+\frac{L}{N}\sum_{t=1}^T\mathbb{E}[\|x_t-\bone\Bar{x}_t\|^2]+\frac{\eta}{2}\mathbb{E}[\|\nabla F(\Bar{x}_{T+1})\|^2]+2\eta\sum_{t=1}^T\mathbb{E}[\|\nabla F(\Bar{x}_t)-\nabla f_t(\Bar{x}_t)\|^2]\\
    &+2\eta\sum_{t=1}^T\mathbb{E}[\|\nabla f_t(\Bar{x}_t)-g_t\|^2]+\frac{3\eta+\eta^2L}{2}\sum_{t=1}^T\mathbb{E}[\|g_t\|^2]\\
    \leq& \frac{\|\Bar{x}_1-x^*\|^2}{2\eta}+\frac{L}{N}\sum_{t=1}^T\mathbb{E}[\|x_t-\bone\Bar{x}_t\|^2]+\frac{\eta}{2}\mathbb{E}[\|\nabla F(\Bar{x}_{T+1})\|^2]+\frac{2\sigma^2\eta T}{N}+2\eta\sum_{t=1}^T\mathbb{E}[\|\nabla f_t(\Bar{x}_t)-g_t\|^2]\\
    &+\frac{3\eta+\eta^2L}{2}\sum_{t=1}^T\mathbb{E}[\|g_t\|^2],
\end{align*}
where the last inequality holds because
\begin{align*}
    \mathbb{E}[\|\nabla f_t(\Bar{x}_t)-\nabla F(\Bar{x}_t)\|^2]=\frac{1}{N^2}\sum_{i=1}^N\mathbb{E}[\|\nabla f_{t,i}(\Bar{x}_t)-\nabla F(\Bar{x}_t)\|^2]\leq \frac{\sigma^2}{N}.
\end{align*}

For the term $\mathbb{E}[\|\nabla f_t(\Bar{x}_t)-g_t\|^2]$, it is clear that
\begin{align}
     \|g_t-\nabla f_t(\Bar{x}_t)\|&=\|\sum_{i=1}^N \frac{\nabla f_{t,i}(x_{t,i})-\nabla f_{t,i}(\Bar{x}_t)}{N}\|\nonumber\\
     &\leq L\sum_{i=1}^N\frac{\|x_{t,i}-\Bar{x}_t\|}{N}\nonumber\\
     &\leq \frac{L}{\sqrt{N}}\|x_t-\bone \Bar{x}_t\|.
 \end{align}
 Therefore,
\begin{align}\label{gt2}
    \mathbb{E}\left[\sum_{t=1}^{T}\left[\|g_t-\nabla f_t(\Bar{x}_t)\|^2\right]\right]&\leq \mathbb{E}\left[\sum_{t=1}^{T}\frac{L^2}{N}\|x_t-\bone\Bar{x}_t\|^2\right]\nonumber\\
    &=\frac{L^2}{N}\mathbb{E}\left[\sum_{t=1}^{T}\|x_t-\bone\Bar{x}_t\|^2\right].
\end{align}

To obtain an upper bound on $\sum_{t=1}^T \mathbb{E}[F(\Bar{x}_t)-F(x^*)]$,  it suffices to bound $\sum_{t=1}^T\mathbb{E}[\|g_t\|^2]$ from above. To this end, based on \eqref{fhat2}, we have
\begin{align*}
    f_t(x_{t,i})\leq& \hat{f}_t+\langle g_t, x_{t,i}-\Bar{x}_t\rangle+L\|x_{t,i}-\Bar{x}_t\|^2+\frac{L}{N}\|x_t-\bone\Bar{x}_t\|^2\\
    \leq& f_t(\Bar{x}_t)+\langle g_t, x_{t,i}-\Bar{x}_t\rangle+L\|x_{t,i}-\Bar{x}_t\|^2+\frac{L}{N}\|x_t-\bone\Bar{x}_t\|^2.
\end{align*}
Therefore,
\begin{align*}
    \mathbb{E}[F(x_{t,i})]\leq \mathbb{E}[F(\Bar{x}_t)]+\mathbb{E}[\langle g_t, x_{t,i}-\Bar{x}_t\rangle]+L\mathbb{E}[|x_{t,i}-\Bar{x}_t\|^2]+\frac{L}{N}\mathbb{E}[\|x_t-\bone\Bar{x}_t\|^2],
\end{align*}
and
\begin{align*}
    &\sum_{t=1}^T [\frac{1}{N}\sum_{i=1}^N \mathbb{E}[F(x_{t+1,i})-F(x^*)]]\\\leq& \sum_{t=1}^T \mathbb{E}[F(\Bar{x}_{t+1})-F(x^*)]+\sum_{t=1}^T\mathbb{E}[\frac{1}{N}\sum_{i=1}^N\langle g_{t+1}, x_{t+1,i}-\Bar{x}_{t+1}\rangle]+\frac{L}{N}\sum_{i=1}^N\sum_{t=1}^{T+1}\mathbb{E}[\|x_{t,i}-\Bar{x}_t\|^2]+\frac{L}{N}\sum_{t=1}^{T+1}\mathbb{E}[\|x_t-\bone\Bar{x}_t\|^2]\\
    \leq&\sum_{t=1}^T \mathbb{E}[F(\Bar{x}_{t+1})-F(x^*)]+\frac{2L}{N}\sum_{t=1}^{T+1}\mathbb{E}[\|x_t-\bone\Bar{x}_t\|^2].
\end{align*}

Based on \eqref{eq:before}, we can obtain that
\begin{align*}
    \sum_{t=1}^T \mathbb{E}[F(\Bar{x}_{t+1})-F(x^*)]\leq \frac{\|\Bar{x}_1-x^*\|^2}{2\eta}+\frac{L}{N}\sum_{t=1}^T \mathbb{E}[\|x_t-\bone\Bar{x}_t\|^2]-(\frac{\eta}{2}-\eta^2L)\sum_{t=1}^T\mathbb{E}[\|g_t\|^2].
\end{align*}
Continuing with $\sum_{t=1}^T [\frac{1}{N}\sum_{i=1}^N \mathbb{E}[F(x_{t+1,i})-F(x^*)]]$, we can have that
\begin{align*}
    &\sum_{t=1}^T [\frac{1}{N}\sum_{i=1}^N \mathbb{E}[F(x_{t+1,i})-F(x^*)]]\\
    \leq& \frac{\|\Bar{x}_1-x^*\|^2}{2\eta}+\frac{3L}{N}\sum_{t=1}^{T+1}\mathbb{E}[\|x_t-\bone\Bar{x}_t\|^2]-(\frac{\eta}{2}-\eta^2L)\sum_{t=1}^T\mathbb{E}[\|g_t\|^2].
\end{align*}
Since $\sum_{t=1}^T [\frac{1}{N}\sum_{i=1}^N \mathbb{E}[F(x_{t+1,i})-F(x^*)]]\geq0$, it follows that
\begin{align*}
    (\frac{\eta}{2}-\eta^2L)\sum_{t=1}^T\mathbb{E}[\|g_t\|^2]\leq \frac{\|\Bar{x}_1-x^*\|^2}{2\eta}+\frac{3L}{N}\sum_{t=1}^{T+1}\mathbb{E}[\|x_t-\bone\Bar{x}_t\|^2].
\end{align*}
For $\eta<\frac{1}{4L}$, it is clear that
\begin{align*}
    \sum_{t=1}^T\mathbb{E}[\|g_t\|^2]\leq \frac{2\|\Bar{x}_1-x^*\|^2}{\eta^2}+\frac{12L}{\eta N}\sum_{t=1}^{T+1}\mathbb{E}[\|x_t-\bone\Bar{x}_t\|^2]
\end{align*}

In a nutshell, we can obtain the upper bound for $R_2$:
\begin{align*}\label{midcentral}
    &\mathbb{E}\left[\sum_{i=1}^N\sum_{t=1}^T f_{t,i}(\Bar{x}_t)-\sum_{i=1}^N\sum_{t=1}^T f_{t,i}(x^*)\right]\nonumber\\
    =&\mathbb{E}\left[N\sum_{t=1}^T \left[f_t(\Bar{x}_t)-f_t(x^*)\right]\right]\nonumber\\
    =&N\mathbb{E}\left[\sum_{t=1}^T[F(\Bar{x}_t)-F(x^*)]\right]\\
    \leq& \frac{N\|\Bar{x}_1-x^*\|^2}{2\eta}+L\sum_{t=1}^T\mathbb{E}[\|x_t-\bone\Bar{x}_t\|^2]+\frac{N\eta}{2}\mathbb{E}[\|\nabla F(\Bar{x}_{T+1})\|^2]+2\sigma^2\eta T+2\eta L^2\mathbb{E}\left[\sum_{t=1}^{T}\|x_t-\bone\Bar{x}_t\|^2\right]\\
    &+\frac{N(3\eta+\eta^2L)}{2}\left(\frac{2\|\Bar{x}_1-x^*\|^2}{\eta^2}+\frac{12L}{\eta N}\sum_{t=1}^{T+1}\mathbb{E}[\|x_t-\bone\Bar{x}_t\|^2]\right)\\
    \leq& \frac{4N\|\Bar{x}_1-x^*\|^2}{\eta}+26L\sum_{t=1}^T\mathbb{E}[\|x_t-\bone\Bar{x}_t\|^2]+\frac{N\eta}{2}\mathbb{E}[\|\nabla F(\Bar{x}_{T+1})\|^2]+2\sigma^2\eta T+24L\mathbb{E}[\|x_{T+1}-\bone\Bar{x}_{T+1}\|^2].
\end{align*}

\end{proof}

\subsection{Proof of Theorem 1}

Based on the analysis of $R_1 $ and $R_2$, we can obtain the regret as follows:
\begin{align*}
    \mathbb{E}[\brn_s]
    =&\mathbb{E}\left[\sum_{i=1}^N\sum_{t=1}^T f_{t,i}(x_{t,i})-\sum_{i=1}^N\sum_{t=1}^T f_{t,i}(x^*)\right]\\
    =&\mathbb{E}\left[\sum_{i=1}^N\sum_{t=1}^T f_{t,i}(x_{t,i})-\sum_{i=1}^N\sum_{t=1}^T f_{t,i}(\Bar{x}_t)\right]+\mathbb{E}\left[\sum_{i=1}^N\sum_{t=1}^T f_{t,i}(\Bar{x}_t)-\sum_{i=1}^N\sum_{t=1}^T f_{t,i}(x^*)\right]\\
    \leq& 28L\sum_{t=1}^T\mathbb{E}[\|x_t-\bone\Bar{x}_t\|^2]+\frac{4N\|\Bar{x}_1-x^*\|^2}{\eta}+\frac{N\eta}{2}\mathbb{E}[\|\nabla F(\Bar{x}_{T+1})\|^2]+2\sigma^2\eta T+24L\mathbb{E}[\|x_{T+1}-\bone\Bar{x}_{T+1}\|^2]\\
    \leq& 28LA_2\eta^2\frac{1+\rho^2}{1-\rho^2}[18\eta^2\sigma^2L^2+18N\eta^2L^2D+(1+\eta LN+N)\sigma^2]T+28LA_1\frac{\alpha-\alpha^T}{1-\alpha}+28L\|x_1-\bone\Bar{x}_1\|^2\\
    &+\frac{4N\|\Bar{x}_1-x^*\|^2}{\eta}+\frac{N\eta}{2}\mathbb{E}[\|\nabla F(\Bar{x}_{T+1})\|^2]+2\sigma^2\eta T+24L\mathbb{E}[\|x_{T+1}-\bone\Bar{x}_{T+1}\|^2]\\
    =& O(\eta^2 T+\eta^2 NT+\frac{N}{\eta}+\eta T)\\
    =& O(\sqrt{NT})
\end{align*}
where  $\eta\leq \frac{1}{2L}\sqrt{\frac{N}{T}}$ and $N=o(T^{1/3})$. Therefore, we conclude that the optimal regret can be achieved and the average regret per agent  $\mathbb{E}[\brn]=\frac{1}{N}\mathbb{E}[\brn_s]\leq O(\sqrt{\frac{T}{N}})$. 

\section{Distributed Convex Stochastic Optimization}

As a byproduct, we can achieve the following convergence guarantee of DOGD-GT for distributed convex stochastic optimization.
\begin{corollary}\label{corollary1}
Suppose Assumptions 1, 2, and 3 hold, and let $\hat{x}_i=\frac{1}{T}\sum_{t=1}^T x_{t,i}$ be the final output of DOGD-GT for each agent $i$. It follows that
{\small
\begin{align*}
\vspace{-0.1cm}
    \frac{1}{N}\sum_{i=1}^N\mathbb{E}[F(\hat{x}_i)-F(x^*)]=O\left(\frac{1}{\sqrt{NT}}\right).
    \vspace{-0.1cm}
\end{align*}
}%
\end{corollary}

\begin{proof}
Based on the convexity of $F(\cdot)$ and Jensen's inequality, we can have
\begin{align*}
    \frac{1}{N}\sum_{i=1}^N \mathbb{E}[F(\hat{x}_i)-F(x^*)]
    \leq& \frac{1}{N}\sum_{i=1}^N \mathbb{E}\left[\frac{1}{T}\sum_{t=1}^T F(x_{t,i})-F(x^*)\right]\\
    =&\frac{1}{N}\sum_{i=1}^N \mathbb{E}\left[\frac{1}{T}\sum_{t=1}^T F(x_{t,i})-\frac{1}{T}\sum_{t=1}^T F(x^*)\right]\\
    =&\frac{1}{T}\mathbb{E}[\brn]\\
    =& O\left(\frac{1}{\sqrt{NT}}\right).
\end{align*}
\end{proof}

Corollary \ref{corollary1} indicates that the optimal  convergence rate of $O(1/\sqrt{NT})$ can be obtained by DOGD-GT for convex stochastic optimization problems. 
In contrast to standard stochastic gradient descent algorithms, it is clear that DOGD-GT can achieve  a factor of $\sqrt{1/N}$ speedup compared with the single-agent case.

\section{Proof of Theorem \ref{theorem2}}

Define the  regret for the network-level OCO about $\phi_{t,n}$ with respect to any reference point $\phi$ as
\begin{align*}
    \brn^{init}(\phi)=\frac{1}{N}\sum_{n=1}^N\sum_{t=1}^T f_{t,n}^{init}(\phi_{t,n})-\frac{1}{N}\sum_{n=1}^N\sum_{t=1}^T f_{t,n}^{init}(\phi),
\end{align*}
and the regret about $v_{t,n}$ with respect to any  reference point $v$ as
\begin{align*}
    \brn^{rate}(v)=\frac{1}{N}\sum_{n=1}^N\sum_{t=1}^T f_{t,n}^{rate}(v_{t,n})-\frac{1}{N}\sum_{n=1}^N\sum_{t=1}^T f_{t,n}^{rate}(v).
\end{align*}
Therefore, according to Theorem \ref{theorem1} in Section \ref{section3}, it follows that
\begin{align*}
    \mathbb{E}[\brn^{init}(\phi^*)]=O\left(\sqrt{\frac{mT}{N}}\right)
\end{align*}
for $\phi^*=\arg\min_{\phi\in\Theta}\mathbb{E}[f_{t,n}^{init}(\phi)]$,
and that
\begin{align*}
    \mathbb{E}[\brn^{rate}(\Tilde{v}^*)]=O\left(\sqrt{\frac{mT}{N}}\right)
\end{align*}
for $\Tilde{v}^*=\arg\min_{v\geq\epsilon}\mathbb{E}[f_{t,n}^{rate}(v)]$.

Based on Theorem 3.1 in \cite{khodak2019adaptive}, we can have
\begin{align*}
    \mathbb{E}[\brn_a] \leq& \mathbb{E}\left[\frac{1}{NT}\sum_{n=1}^N\sum_{t=1}^T \left(\frac{\mathcal{B}_R(\theta^*_{t,n}||\phi_{t,n})}{v_{t,n}}+v_{t,n}\right)G\sqrt{m}\right]\\
    \leq& \frac{1}{T}\left\{\mathbb{E}[\brn^{rate}(\Tilde{v}^*)]+\min_v\mathbb{E}\left[\frac{1}{N}\sum_{n=1}^N\sum_{t=1}^T \left(\frac{\mathcal{B}_R(\theta^*_{t,n}||\phi_{t,n})}{v}+v\right)G\sqrt{m}\right]\right\}\\
    \leq& \frac{\mathbb{E}[\brn^{rate}(\Tilde{v}^*)]}{T}+\min_v\frac{1}{T}\left\{\frac{\mathbb{E}[\brn^{init}(\phi^*)]}{v}+\mathbb{E}\left[\frac{1}{N}\sum_{n=1}^N\sum_{t=1}^T \left(\frac{\mathcal{B}_R(\theta^*_{t,n}||\phi^*)}{v}+v\right)G\sqrt{m}\right]\right\}\\
    \overset{(a)}{\leq}& \frac{\mathbb{E}[\brn^{rate}(\Tilde{v}^*)]}{T}+\frac{1}{T}\min\left\{\frac{\mathbb{E}[\brn^{init}(\phi^*)]}{V_{\phi}},2\sqrt{\mathbb{E}[\brn^{init}(\phi^*)]GT\sqrt{m}}\right\}+2V_{\phi}GT\sqrt{m}\\
    =& O\left(\frac{1+\frac{1}{V_{\phi}}}{\sqrt{NT}}+V_{\phi}\right)\sqrt{m}
\end{align*}
where (a) is true for $V_{\phi}=\sqrt{\mathbb{E}[\mathcal{B}_R(\theta^*_{t,n}||\phi^*)]}$ and $v=\max\left\{V_{\phi},\sqrt{\frac{\mathbb{E}[\brn^{init}(\phi^*)]}{GT\sqrt{m}}}\right\}$.

\end{document}

%% file: math_def.tex
\def \n2{{N_0 \over 2}}

\def \bbr{\bar{\mathbf{R}}}
\def \brh{\hat{\mathbf{R}}}
\def \brn{\mathbf{R}}

\def\det{\mbox{\textrm{det}}}
\def\adj{\mbox{\textrm{adj}}}

\def \h5{\hspace{0.5in}}

\def \bone{\mbox{\boldmath $1$}}